\documentclass[11pt]{tex/llncs}
\usepackage[letterpaper,left=0.75in,right=0.75in,top=1in,bottom=1in,
			footskip=.25in]{geometry}

\newif\ifdraft
\draftfalse

\newif\iffull
\fulltrue

\iffull
\fi

\newif\ifoc 
\ocfalse

\ifdraft
\usepackage{xcolor}
\usepackage{xargs} 
\usepackage{tikz} 
\usepackage[textsize=footnotesize]{todonotes}
\newcommandx{\wt}[2][1=]{\todo[linecolor=blue,
			backgroundcolor=blue!10,bordercolor=blue,#1]{{\bf Wei}: #2}}
\newcommandx{\jy}[2][1=]{\todo[linecolor=red,
			backgroundcolor=red!10,bordercolor=red,#1]{{\bf JJ}: #2}}
\else
\def\wt#1{}
\def\jy#1{}
\fi

\usepackage{mathptmx}
\usepackage{tabularx}
\usepackage{amssymb,amsmath}
\usepackage{epsfig}
\usepackage{textcomp}
\usepackage{epstopdf}
\usepackage{url}
\usepackage{wrapfig}
\usepackage{enumerate}
\usepackage{soul}
\usepackage{tipa}
\usepackage{cite}
\usepackage[percent]{overpic}
\usepackage{tikz}
\usepackage{pgfgantt}
\usepackage{xspace}
\usepackage{comment}

\usepackage{enumitem}

\usepackage[linesnumbered,ruled]{algorithm2e}
\SetKwInput{KwIn}{Initial condition}
\SetKwInput{KwResult}{Outcome}
\SetAlgorithmName{Algorithm}{List of algorithms}

\makeatletter
\newcommand{\customlabel}[2]{%
\protected@write \@auxout {}{\string \newlabel {#1}{{#2}{}}}}
\makeatother

\def\remark{\noindent\textbf{Remark}.\xspace}

\usepackage{mathtools,xparse}

\def\crp{{\tt {CRP}}\xspace}
\def\sat{{\tt {3-SAT}}\xspace}
\def\mpsat{{\tt {MPSAT}}\xspace}
\def\namo{{\tt {NAMO}}\xspace}
\def\tamp{{\tt {TAMP}}\xspace}

\def\rrts{RRT$^*$\xspace}

\def\W{\mathcal W}

\def\pl{\prec^{\mathrm{L}}}
\def\pr{\prec^{\mathrm{R}}}
\def\sl{\succ^{\mathrm{L}}}
\def\sr{\succ^{\mathrm{R}}}

\font\titlefont=ptmb at 17.2pt
\title{{\titlefont Taming Combinatorial Challenges in Optimal Clutter Removal Tasks}}
\author{Wei N. Tang \qquad Jingjin Yu}
\institute{Department of Computer Science, Rutgers University \\ 
\{wei.tang, jingjin.yu\}@rutgers.edu}
\begin{document}
\maketitle
\iffull\else\vspace*{-6mm}\fi
\begin{abstract}
We examine an important combinatorial challenge in clearing clutter 
using a mobile robot equipped with a manipulator, seeking to compute 
an optimal object removal sequence for minimizing the task completion 
time, assuming that each object is grasped once and then subsequently 
removed. On the structural side, we establish that such an optimal 
sequence can be NP-hard to compute, even when no two objects to be 
removed have any overlap. Then, we construct asymptotically optimal 
and heuristic algorithms for clutter removal. Employing dynamic 
programming, our optimal algorithm scales to 40 objects. On the other 
hand, for random clutter, fast greedy algorithms tend to produce 
solutions comparable to these generated by the optimal algorithm. 
\end{abstract}

\iffull\else\vspace*{-8mm}\fi
\section{Introduction}\label{section:introduction}
\iffull\else\vspace*{-2mm}\fi
We investigate the challenge of clearing clutter with a mobile robot, 
as an initial step toward the autonomous execution of clean-up tasks, 
e.g., the handling of the aftermath of earthquakes in urban environments
or the tidying up of the daily mess in a kid's room. Specifically, the 
study focuses on the task and motion planning (\tamp) for removing scattered 
cuboid-like objects with known poses, in a bounded 2D region with exit(s),
where each object is grasped once and subsequently removed. We call this
the clutter removal problem (\crp) and our main goal is to design 
effective algorithms for computing high quality object removal sequences 
for minimizing the overall task completion time. A typical setting 
examined in this paper is illustrated in Fig.~\ref{fig:scene}. 

Due to the extremely high complexity of clutter removal as a general 
\tamp task, we explicitly note that the current work has a limited 
scope on the stated combinatorial challenge and does not consider 
other important issues such as uncertainties rising from perception 
or non-prehensile manipulation. Nevertheless, the addressed problem 
remains relevant when other factors are considered; therefore, the 
results we provided in this paper has general applicability. 
Reasoning about the inherent constraints associated with the challenge 
including objects' shapes, poses (location and orientation), and their 
relative placement with respect to each other, we are able to establish 
that finding optimal plans for \crp is an intractable task, even when 
objects assume a planar setting. 
%
On the algorithmic side, first, for the single-exit case, we develop an 
backtracking-based asymptotically optimal algorithm for solving \crp, 
capable of handling 40 objects, which is fairly practical. Then, multiple 
sub-optimal, best-first type algorithms are developed that perform very 
well under practical settings. Building on the single-exit solution, we 
further develop Voronoi-based algorithms for the case of multiple exits
that achieve both high solution quality and decent computational 
efficiency. 

\begin{figure}[ht!]
\begin{center}
\begin{overpic}[width=4in,tics=5]{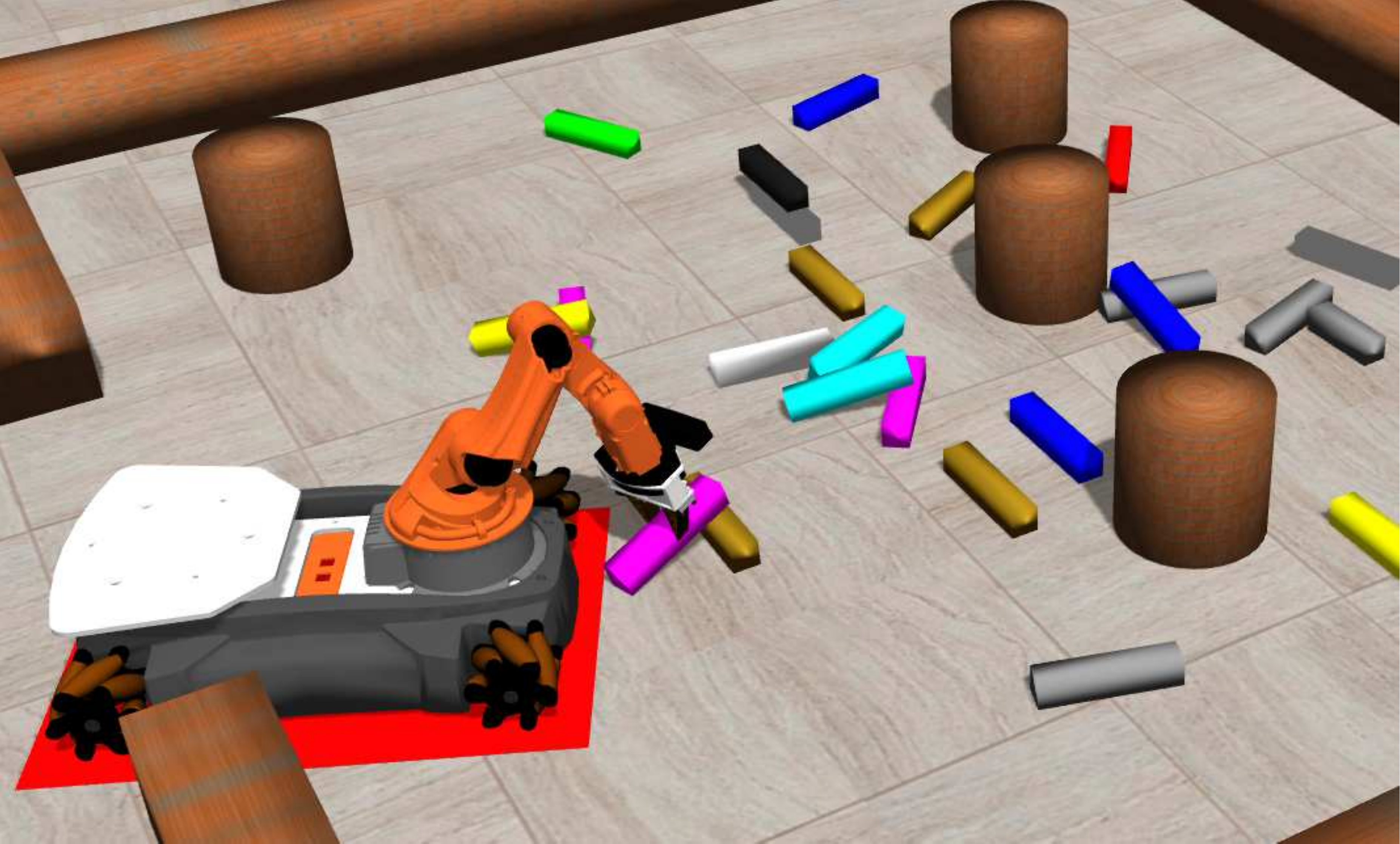}
\end{overpic}
\end{center}
\iffull\else\vspace*{-5mm}\fi
\caption{\label{fig:scene} A clutter removal scenario 
addressed in this work, where a KUKA youBot is tasked to grasp and
remove all scattered objects in clutter, one at a time, in a room with 
four static obstacles.}
\iffull\else\vspace*{-7mm}\fi
\end{figure} 

Our study is mainly motivated by the need and potential of deploying 
autonomous robots in disaster response scenarios \cite{pratt2013darpa,
murphy2014disaster}. The realization of this goal demands the efficient 
resolution of \tamp challenges \cite{plaku2010sampling,garrett2015ffrob,
dantam2016incremental}. These \tamp challenges, involving both discrete 
combinatorial reasoning and (continuous) motion planning, can often be 
notoriously hard to solve. For example, a class of problems related to 
this work, {\em Navigation among Movable Obstacles} (\namo), are known 
to be computationally intractable in many forms \cite{wilfong1991motion,
nieuwenhuisen2008effective}. Nevertheless, practical algorithms have 
been proposed that effectively solve the {\em monotone} case (i.e., where 
a solution exists that requires moving each obstacle once) via standard 
backtracking techniques \cite{stilman2008planning}. Probabilistically 
complete solutions for general settings have also been proposed 
\cite{van2009path}. The current study emphasizes optimality issues in 
\crp as a \tamp problem similar to \cite{HanStiKroBekYu18IJRR,
vega2016asymptotically}. This contrasts studies with integrated \tamp 
solutions, e.g., \cite{plaku2010sampling,garrett2015ffrob,
dantam2016incremental} which do not provide optimality assurances. 

{\em Object Rearrangement} is another related problem class.
Some results in this area, e.g., \cite{ota2009rearrangement,
havur2014geometric}, can be viewed as variations of \namo. Whereas a 
search based approach is used in \cite{ota2009rearrangement}, symbolic 
reasoning is applied in \cite{havur2014geometric} which appears to be 
more general. In contrast, \cite{krontiris2015dealing,HanStiKroBekYu18IJRR} 
put more emphasis on taming the combinatorial explosion caused by the sheer 
number of objects involved, with \cite{HanStiKroBekYu18IJRR} further 
computing (near) optimal solutions under a metric considering both grasping 
costs and end-effector travel costs. 

Clutter removal is also intimately linked to {\em (dis)assembly}, 
where multiple parts need to be put together to yield a product, 
e.g., \cite{natarajan1988planning,lozano1993assembly}. The (dis)assembly 
problem is hard in general \cite{kavraki1993complexity} and remains so even 
if the parts are put together two at a time \cite{kavraki1995partitioning}. 
From the algorithmic perspective, planning of (dis)assembly algorithms is 
studied in \cite{wilson1994geometric}, which also proposed measures for 
evaluating the complexity of the resulting algorithms. Subsequently, a more 
general motion space approach was developed \cite{halperin2000general}, 
which proposed a block graph abstraction for representing dependency between 
components.

The main contributions of the work are three-fold. 
First, we show that computing an optimal sequence for \crp, as 
a fundamental discrete combinatorial \tamp challenge independent of the 
geometrical grasp and motion planning components, is NP-hard. 
Second, we develop a sampling-based, asymptotically optimal algorithm 
for \crp. Employing dynamic programming and other techniques, the 
algorithm is capable of handling up to 40 objects given limited computation 
time. This is significant because there are $n!$ possible sequences to 
consider for $n$ objects. 
Third, we continue to develop fast best-first type algorithms that are 
empirically shown to compute near-optimal solutions under randomized 
settings to be expected in real-world scenarios. 

The rest of the manuscript is organized as follows. 
Section~\ref{section:problem} describes the clutter removal problem (\crp)  
studied in this paper. 
Section~\ref{section:analysis} then provides some initial analysis regarding 
feasibility and completeness, and outlines the general algorithmic solution.
Section~\ref{section:complexity} shows that the combinatorial elements of 
optimal clutter removal is NP-hard even when the objects do not overlap, 
i.e., in a planar setting. 
Section~\ref{section:algorithm} proceeds to develop resolution-complete 
asymptotically optimal algorithms as well as fast best-first algorithms, 
and shows that there are cases where greedy algorithms yield rather 
sub-optimal solutions. 
Section~\ref{section:evaluation} evaluates the performance of the proposed 
algorithms on computation effort and solution optimality. 
Section~\ref{section:conclusion} concludes the work.

%

\iffull\else\vspace*{-3mm}\fi
\section{The Clutter Removal Problem}\label{section:problem}
\iffull\else\vspace*{-3mm}\fi
Consider the setting in which $n$ rigid objects $O = \{o_1, \ldots, o_n\}$ 
are scattered on the ground of a bounded 3D workspace, with $o_i$ 
representing the known pose (i.e., location and orientation) 
of the $i$-th object, $1 \le i \le n$. Let $\W \subset \mathbb R^2$ be the 
ground plane of the workspace, which may also contain static obstacles, 
i.e., inaccessible regions. Let $\partial \W$ be the boundary of $\W$. The 
workspace can be accessed through {\em exits} along $\partial \W$ by a 
mobile robot capable of grasping and transporting objects, one at a time. 
The task in a clutter removal problem (\crp) is to remove all objects 
from the workspace. An object is considered cleared after it is carried by 
the robot outside an exit. Initially, the mobile robot starts at a specific 
exit. The robot may travel between exits along $\partial \W$. 
We note that, due to inherent limitations of mobile robots and the 
placement of the objects (e.g., an object $o_j$ close to another object 
$o_i$ may prevent the robot from successfully grasping $o_i$), some 
objects may be inaccessible to the robot at any given time. 
Fig.~\ref{fig:cr2d} illustrates the top view of a problem instance 
with static obstacles and three exits. 

\begin{figure}[ht!]
\iffull\else\vspace*{-4mm}\fi
\begin{center}
\begin{overpic}[width=4in,tics=5]{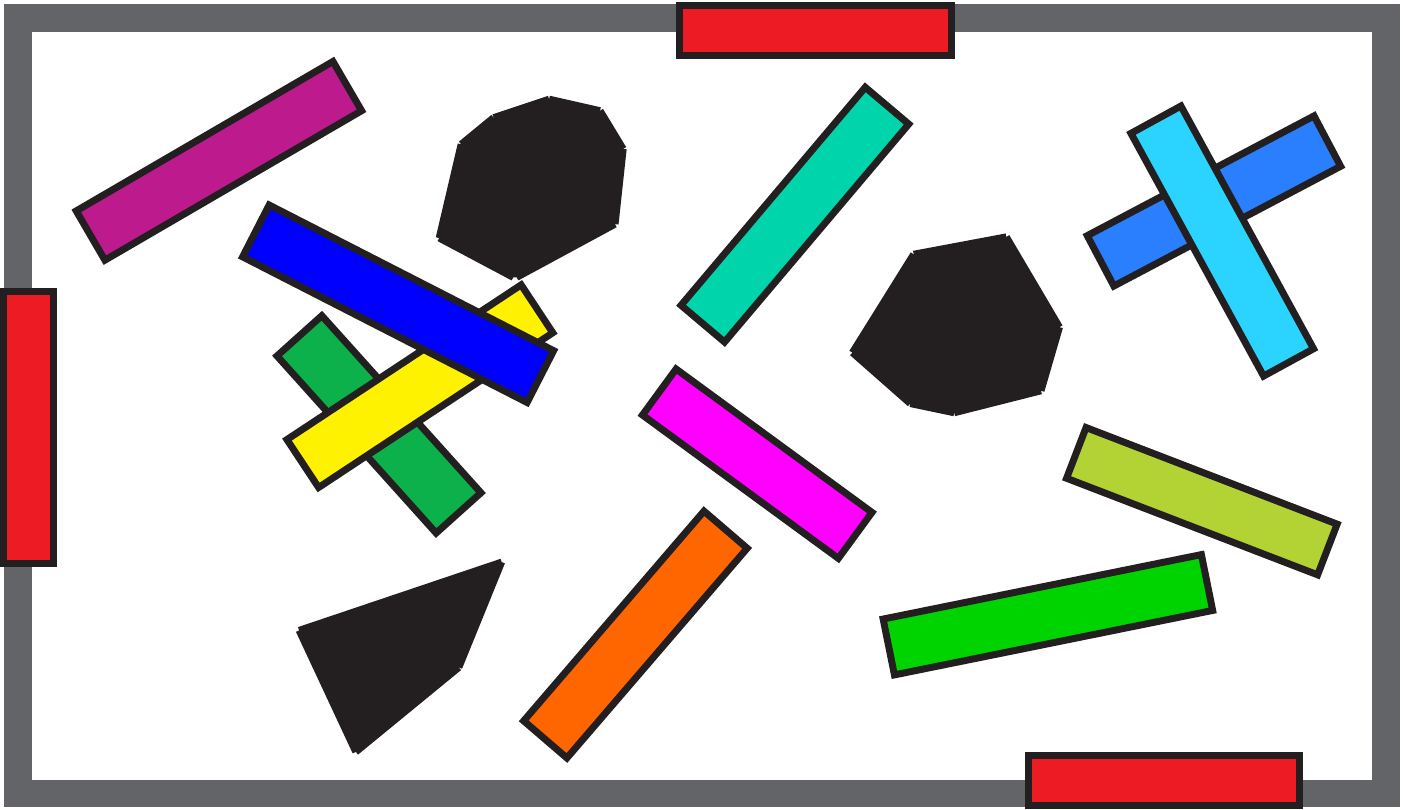}
\end{overpic}
\end{center}
\iffull\else\vspace*{-4mm}\fi
\caption{\label{fig:cr2d} An example \crp in 2D (viewed from above) where 
$\W$ is the region within the large rectangle. Three exits are marked red 
on $\partial \W$. The black polygons inside $\W$ are static obstacles (in 
additional to $\partial \W$). 
The rest of the objects, each simplified as a rectangle, are to be 
removed.}
\iffull\else\vspace*{-6mm}\fi
\end{figure} 

Focusing on the combinatorial problem of computing the optimal clutter 
removal sequence and given the extreme complexity of the general \tamp 
problem, this work does not consider non-prehensile manipulation, grasping 
failures, object pose uncertainty, or multiple grasps per object. 
Under these assumptions, we work with cuboid-like objects with known 
poses, for which we may assume that it takes the same amount of time 
to grasp an object at different poses. Subsequently, solving the \crp 
formulation optimally reduces to computing a clutter removal sequence 
to minimize the travel time of the robot. This remains highly challenging 
because: {\em (i)} grasp planning and robot base motion planning must be 
performed continuously to reason about object accessibility and how they 
can be removed, and {\em (ii)} given $n$ objects, there are $n!$ possible 
sequences with which they may be removed; any optimal solution must 
consider every one of these removal sequences during its computation
phase.

\iffull\else\vspace*{-3mm}\fi
\section{Preliminary Structural Analysis and Algorithm Design}\label{section:analysis}
\iffull\else\vspace*{-3mm}\fi
\subsection{Feasibility and Completeness}
\iffull\else\vspace*{-3mm}\fi
To solve a \crp instance and obtain an object removal sequence, one 
must first identify at any point the current set of graspable objects. 
Then, one of the accessible objects must 
be removed and the process repeats. One of the first issue here is 
whether an algorithm we design needs to be careful so that an initially 
feasible problem is not made infeasible. We make the observation that 
when non-prehensile manipulation is not considered, a feasible \crp 
instance will remain feasible regardless of the object removal order. 

\begin{proposition}\label{p:complete}
\iffull\else\vspace*{-1mm}\fi
Adopting a proper (resolution-)complete motion planning algorithm, the 
clutter removal problem, in the absence of interactions among objects, 
can be solved with (resolution) completeness guarantees.
\iffull\else\vspace*{-1mm}\fi
\end{proposition}
\begin{proof}
No explicit requirement is placed on the feasibility of a \crp instance. 
However, we note that, to be able to remove all objects sequentially, 
there must exist at least one ordering of the $n$ objects with which they 
{\em can} be removed one by one. If such an order does exist, since 
non-prehensile manipulation is not considered in this study, i.e., 
grasping an object will not make another object accessible to the 
robot become inaccessible, this implies the existence of a feasible 
solution for removing objects regardless of the actual object removal 
order. Therefore, if the initial problem admits a solution, then, at 
any stage, some object can be removed from the workspace. Subsequently, 
using a complete \cite{canny1988complexity} (resp., resolution-complete 
\cite{kavraki1994probabilistic,lavalle2001randomized,karaman2011sampling}) 
motion planning algorithm can guarantee the completeness (resp., 
resolution completeness).~\qed
\end{proof}

\iffull\else\vspace*{-4mm}\fi
\subsection{Algorithm Structure and Common Routines}\label{sec:alg-comp}
\iffull\else\vspace*{-2mm}\fi
With Proposition~\ref{p:complete} characterizing the feasibility and 
completeness for clutter removal, we shift the attention to algorithm 
design. Since the objects must be grasped and removed one by one, we 
need subroutines for computing the current set of graspable objects
and the shortest distance to reach these objects. 

\iffull\else\vspace*{-3mm}\fi
\subsubsection{Motion planning for the robot base.} Motion planning for 
the mobile robot (base) is carried out for two purposes: to compute 
optimal trajectories and to identify objects within the robot's 
reach. In this work, these are achieved using a variant of the RRT$^*$ 
algorithm \cite{karaman2011sampling}, augmented with an updating 
heuristic proposed as part of RRT$^X$ \cite{otte2015mathrm}.

The cluttered objects (with known poses) are projected over $\W$ (recall 
that $\W \subset \mathbb R^2$ is the 2D projection of the 3D work 
space onto the ground plane). Treating the projection and the workspace 
boundary (i.e. $\partial \W$ without the exits) as obstacles, one or 
more \rrts structures can be computed. An illustration of the \rrts 
structure (for a single exit) before and after an object removal is 
shown in Fig.~\ref{fig:rrts}. In the \rrts update, the idea of {\em 
cascade rewiring} with a larger radius \cite{otte2015mathrm} is adopted, 
which use all relevant existing \rrts samples and add new ones only in 
the area of the newly removed object. 

\begin{figure}[ht!]
\begin{center}
\begin{overpic}[width=4in,tics=5]{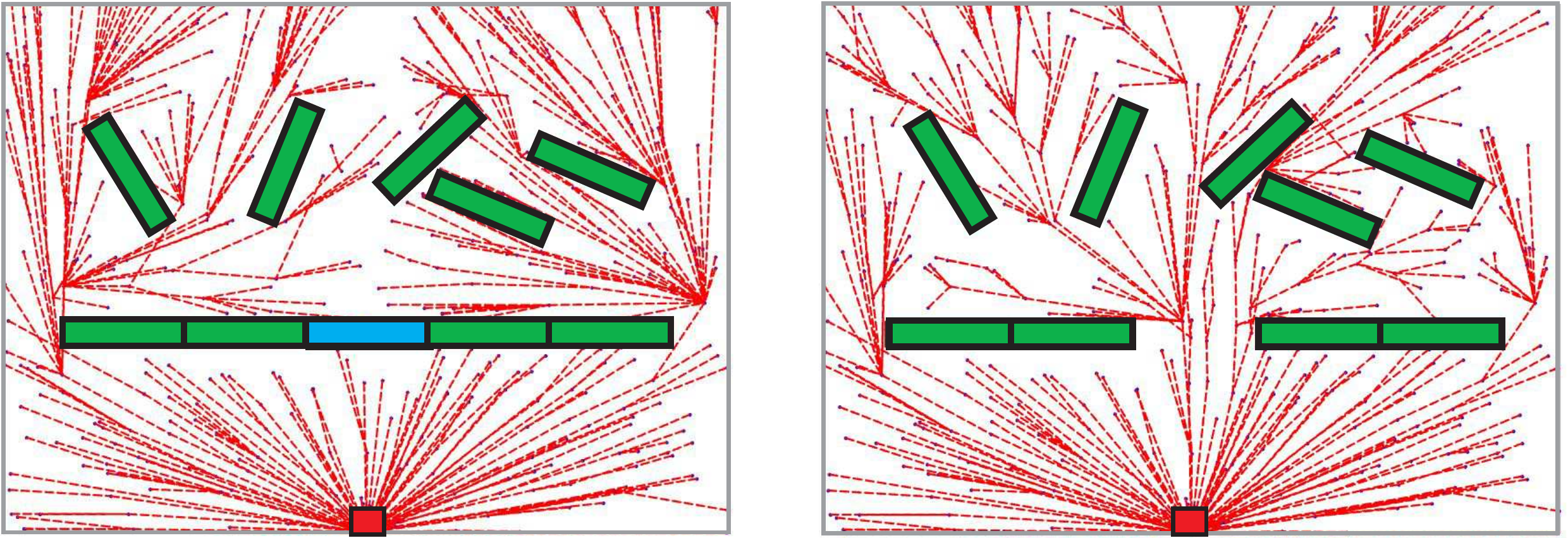}
\end{overpic}
\end{center}
\iffull\else\vspace*{-4mm}\fi
\caption{\label{fig:rrts} An illustration of the maintenance of the \rrts 
structure before and after the removal of an object (the cyan one on the 
left). The small red rectangles at the bottom indicate the (single) 
exit, the green rectangles with black borders are the rest of the objects 
to be removed.}
\iffull\else\vspace*{-6mm}\fi
\end{figure} 

We mention here that for symmetric omnidirectional robot, it is also 
possible to use the {\em visibility graph} \cite{lozano1979algorithm} 
to compute optimal trajectories when a polygonal approximation of the 
2D projection can be obtained, which can be much faster. 

\begin{wrapfigure}[6]{r}{1.25in}
\iffull\vspace*{-7mm}\else\vspace*{-8mm}\fi
\begin{overpic}[width=1.2in,tics=5]{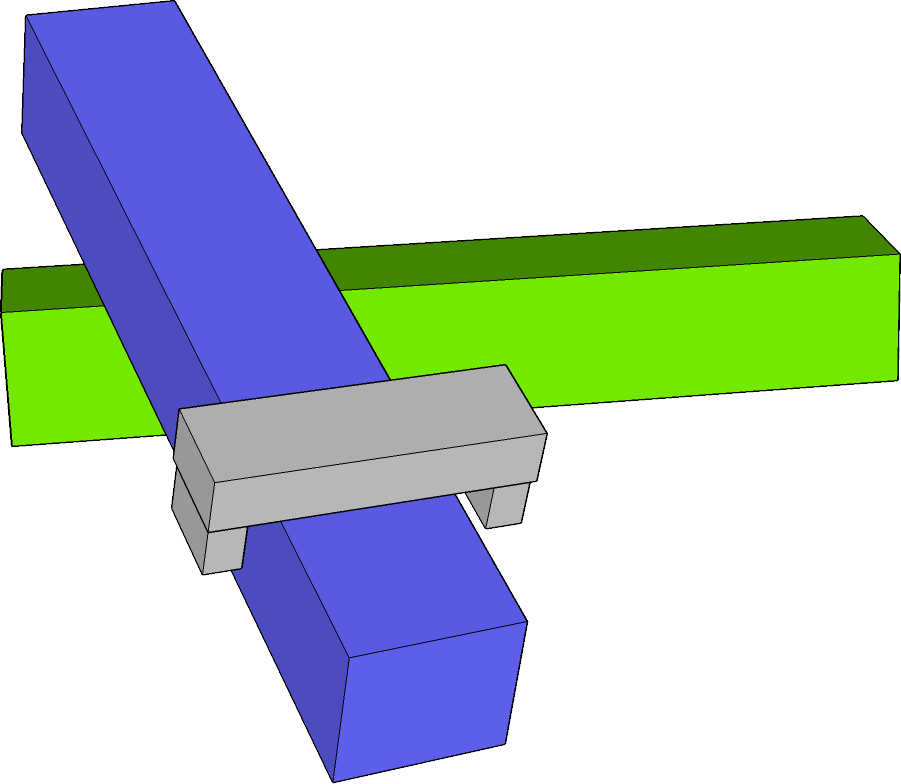}
\end{overpic}
\end{wrapfigure}
\iffull\else\vspace*{-3mm}\fi
\subsubsection{Grasp planning.} With the \rrts computed for the current
environment, reachable objects can be identified. For all these objects, 
a grasp planner is invoked to compute potential grasps. Per the assumption 
that the objects to be removed are cuboid-like, a relatively basic 
grasp planner is applied: for each potential object, the planner first finds
a top face (i.e., one with surface normals pointing up) and samples 
the normals for possible grasps by a 2-finger gripper. In the figure on the
right, the gray block illustrates a possible sampled grasp for an accessible 
object identified by the planner. 

\iffull\else\vspace*{-3mm}\fi
\subsubsection{General algorithm structure.}
Based on the sampled grasps returned by the grasp planner and the \rrts, 
travel distance costs for reaching the grasp by the robot can be computed 
accordingly. We note that as we increase the resolution of the two 
sampling process, the costs that are computed will be asymptotically 
optimal. With these costs, what is left is the computation of an object 
removal sequence. As such, all the algorithms proposed in this paper 
share common grasp planning and robot-base motion planning subroutines, 
and differ on how they use the information returned by the subroutines 
to compute the object removal sequence, where there are up to $n!$ choices.

\iffull\else\vspace*{-3mm}\fi
\section{Hardness of Optimal Clutter Removal}\label{section:complexity}
\iffull\else\vspace*{-3mm}\fi
\iffull
Before constructing full algorithms for \crp, we establish that 
computing the optimal object removal sequence is computationally 
intractable, even when objects to be removed do not overlap, i.e., 
the setting is {\em planar}. In this section, NP-hard is shown for 
cases with two or more exits. The proof for the single exit case 
is significantly more involved and is provided in the Appendix.
%
\else
Before constructing full algorithms for \crp, we establish that 
computing the optimal object removal sequence is computationally 
intractable, even when objects to be removed do not overlap, i.e., 
the setting is {\em planar}. In this section, NP-hard is shown for 
cases with two or more exits. The proof for the single exit case 
is significantly more involved and is provided as supplementary 
material.
%
\fi

Our proof of the hardness result is via a reduction from {\em monotone 
planar} \sat \cite{de2010optimal}, with the help of some special {\em 
gadgets}. In the construction, we assume that the robot is omnidirectional
and powerful enough to grasp and transport large objects 
providing that the object has suitable graspable {\em handles}. 

\subsection{Monotone Planar \sat}\label{subsection:mpsat}
\iffull\else\vspace*{-3mm}\fi
Monotone planar \sat (\mpsat) is a variation of \sat 
\cite{garey2002computers} with three additional restrictions: 
{\em (i)} each clause contains exclusively positive literals or 
exclusively negative literals, {\em (ii)} the graph connecting clauses 
to literals has a planar embedding, and {\em (iii)} the planar 
embedding can be arranged such that positive clauses and negative 
clauses reside on two sides of a line connecting all the variables. 
As an illustration, Fig.~\ref{fig:mpsat} provides a planar embedding for 
the \mpsat instance with variables $x_1$--$x_5$ and clauses $c_1 = x_1\vee 
x_2, c_2 = x_1 \vee x_3 \vee x_4, c_3 = x_1 \vee x_4 \vee x_5, c_4 = \neg 
x_1 \vee \neg x_2 \vee \neg x_3$, and $c_5 = \neg x_3 \vee \neg x_4 \vee 
\neg x_5$. We will be using this example for illustrating the NP-hardness 
reduction to planar optimal clutter removal. 

\begin{figure}[ht!]
\iffull\else\vspace*{-4mm}\fi
\begin{center}
\begin{overpic}[width=3in,tics=5]{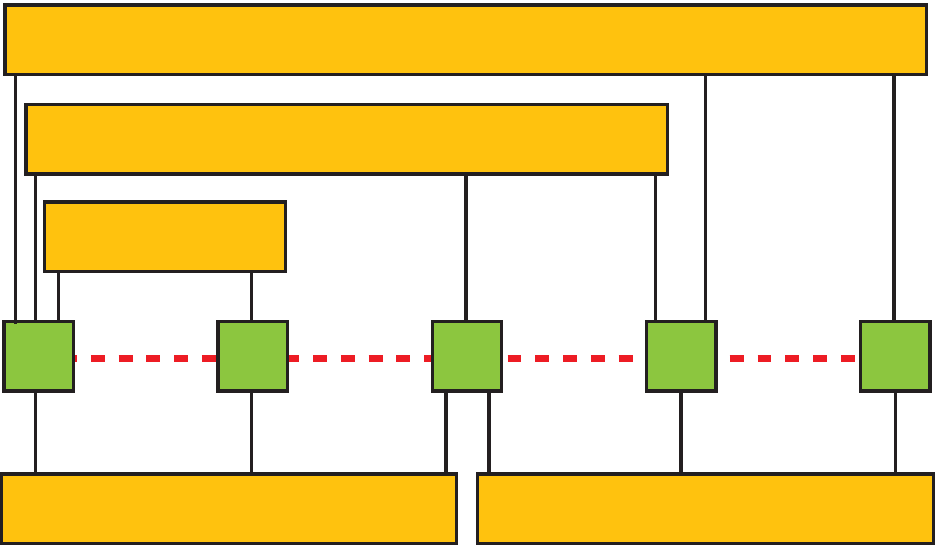}
\put(2,18.8){{\small $x_1$}}
\put(24.5,18.8){{\small $x_2$}}
\put(47,18.8){{\small $x_3$}}
\put(70.5,18.8){{\small $x_4$}}
\put(93,18.8){{\small $x_5$}}
\put(34,52.5){{\small $c_3: x_1 \vee x_4 \vee x_5$}}
\put(20,42){{\small $c_2: x_1 \vee x_3 \vee x_4$}}
\put(7.2,32){{\small $c_1: x_1 \vee x_2$}}
\put(4,2.5){{\small $c_4:\neg x_1 \vee \neg x_2 \vee \neg x_3$}}
\put(55,2.5){{\small $c_5:\neg x_3 \vee \neg x_4 \vee \neg x_5$}}
\end{overpic}
\end{center}
\iffull\else\vspace*{-4mm}\fi
\caption{\label{fig:mpsat}The planar embedding of the \mpsat instance with 
variables $x_1$--$x_5$ and clauses $c_1 = x_1\vee x_2, c_2 = x_1 \vee x_3 
\vee x_4, c_3 = x_1 \vee x_4 \vee x_5, c_4 = \neg x_1 \vee \neg x_2 \vee 
\neg x_3$, and $c_5 = \neg x_3 \vee \neg x_4 \vee \neg x_5$.}
\iffull\else\vspace*{-8mm}\fi
\end{figure}
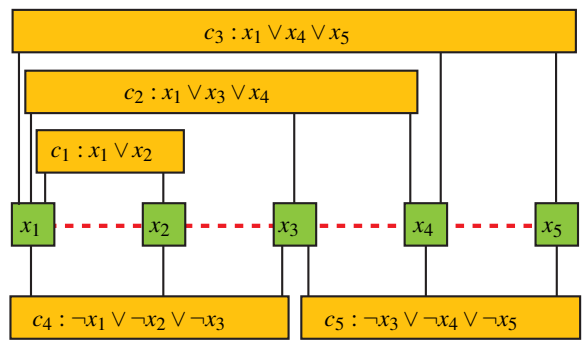 

\iffull\else\vspace*{-2mm}\fi
\subsection{The Variable Gadget}
\iffull\else\vspace*{-2mm}\fi
For each variable in a given \mpsat instance, we build a gadget for it; 
Fig.~\ref{fig:gadgets}(a) illustrates such a gadget for variable $x_1$. 
The boundaries of individual objects are marked with black lines. Here, 
the green object $o_{g}$, representing assigning $x_1$ to be true, can 
be lifted at either its left most part or its lowest part (when fully 
exposed), as indicated by the red arrows. The orange object $o_o$ mirrors 
$o_{g}$ and represents assigning $x_1$ to false. After either $o_{g}$ 
or $o_o$ is removed, the top purple object $o_p$ can be removed. 
The partially shown (three) lime objects (call these $o_{l1}, o_{l2}, 
o_{l3}$) and (one) yellow object $o_y$ are long rectangles representing 
connections between the variable gadget and clause gadgets (to be 
detailed soon in Section~\ref{subsection:clause-gadget}), corresponding 
to the vertical lines shown in Fig.~\ref{fig:mpsat}. In this particular 
case, they are connected to the clause gadgets for $c_1$--$c_4$. 

\begin{figure}[ht!]
\begin{center}
\iffull\else\vspace*{-5mm}\fi
\begin{tabular}{cc}
\begin{overpic}[width=3.5in,tics=5]{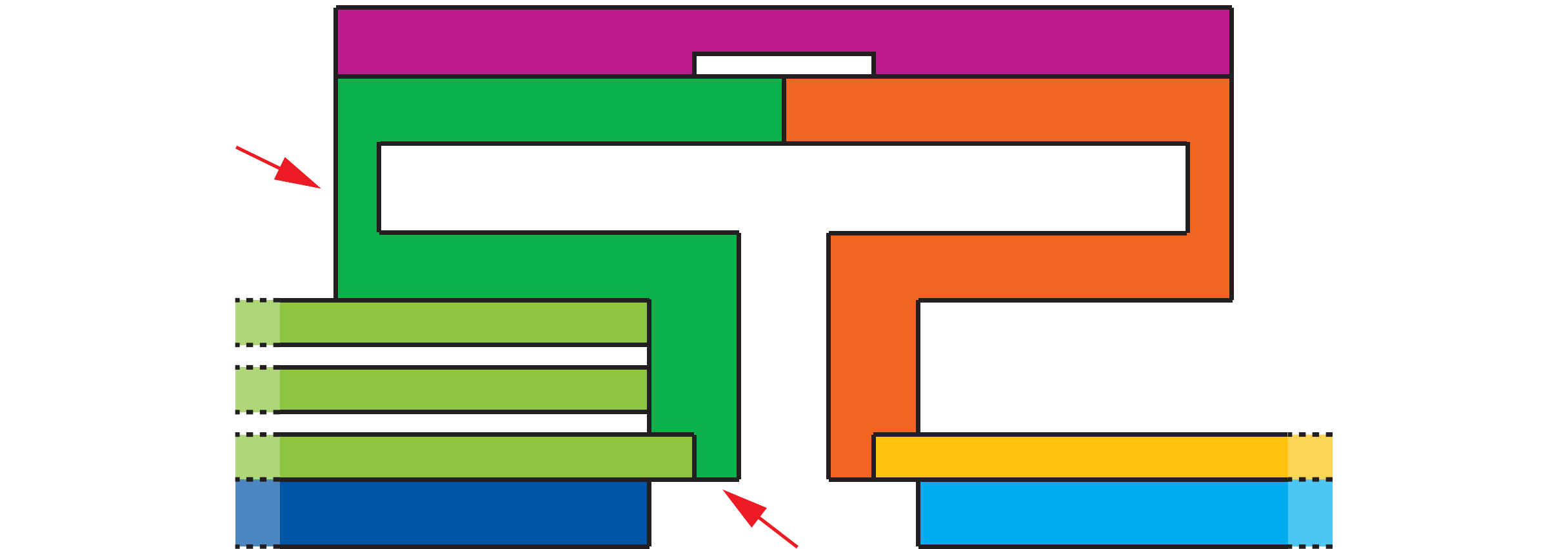}
\put(28,22.5){{\small $o_{g}$}}
\put(68,22.5){{\small $o_{o}$}}
\put(81,31){{\small $o_{p}$}}
\put(-2,13.5){{\small $c_1 \leftarrow o_{l1}$}}
\put(-2,9){{\small $c_2 \leftarrow o_{l2}$}}
\put(-2,4.5){{\small $c_3 \leftarrow o_{l3}$}}
\put(9,0.5){{\small $o_b$}}
\put(86,4.5){{\small $o_y \rightarrow c_4$}}
\put(86,0.5){{\small $o_c$}}
\put(48,-6){{\small (a)}}
\end{overpic} &
\begin{overpic}[width=2.8in,tics=5]{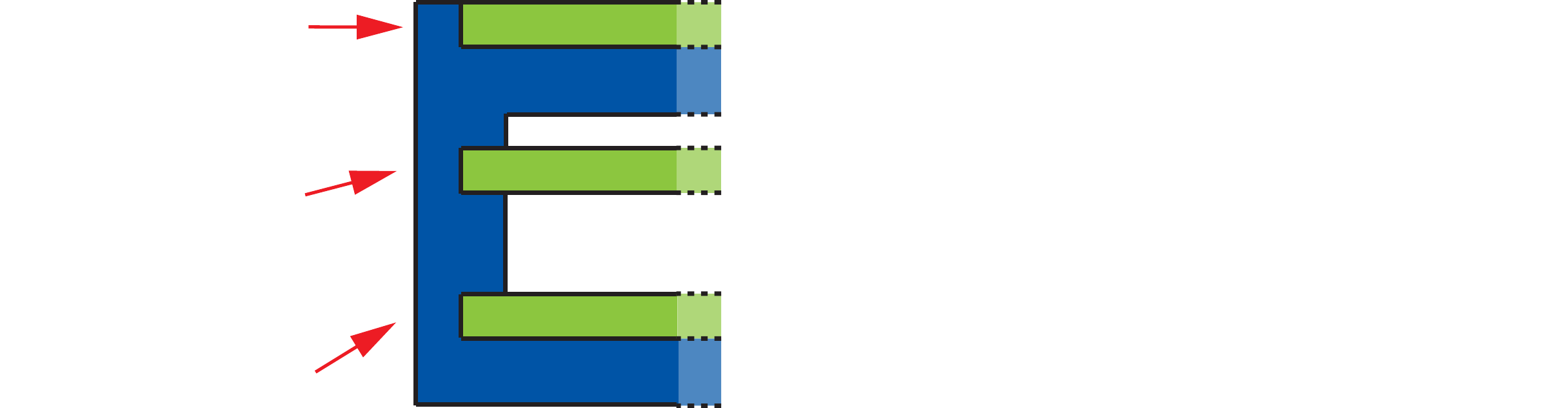}
\put(50,23){{\small $\rightarrow x_5$}}
\put(50,14){{\small $\rightarrow x_4$}}
\put(50,4){{\small $\rightarrow x_1$}}
\put(40,-6){{\small (b)}}
\end{overpic} 
\end{tabular}
\end{center}
\iffull\else\vspace*{-3mm}\fi
\caption{\label{fig:gadgets} (a) The variable gadget for $x_1$, which 
appears positively in $c_1$--$c_3$ and negatively in $c_4$. (b) The 
clause gadget for $c_3$, which connects the variable gadgets for $x_1$, 
$x_4$, and $x_5$. Notice that the figures are not drawn to scale.}
\iffull\else\vspace*{-7mm}\fi
\end{figure}
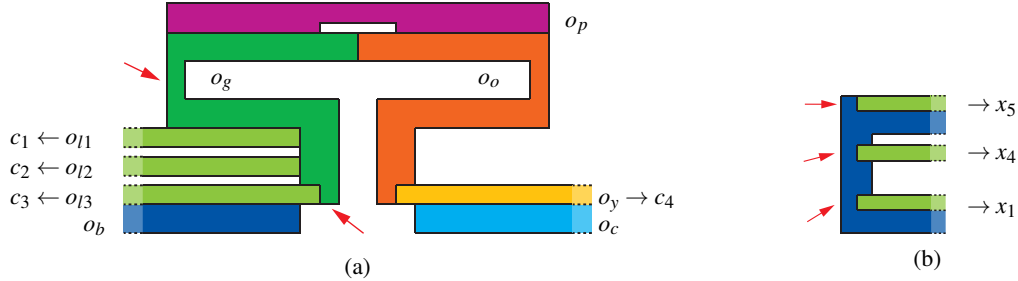 

On the positive side, once $o_{g}$ is removed, in addition to being able
to remove $o_{p}$, $o_{l1}$--$o_{l3}$ can be removed by grasping them 
from their rightmost locations. The blue, partially shown object $o_b$ 
belongs to the clause gadget for $c_3$. Alternatively, if $o_b$ is 
removed, then $o_{l3}$ can be removed from the right side. Following 
this, $o_g$ can be lifted at its lowest point and removed. $o_{l1}$ and 
$o_{l2}$ can be subsequently removed as well. On the negative side, 
because $x_1$ only appears negatively in $c_4$, there is a single 
yellow object $o_y$ connecting $o_o$ to the gadget for $c_4$, which 
contains the cyan object $o_c$. We note that the figure is not drawn to 
scale. The horizontal span (i.e., the width) of the gadget is much 
larger than its vertical span. This will be quantified later. 

\iffull\else\vspace*{-3mm}\fi
\subsection{The Clause Gadget}\label{subsection:clause-gadget}
\iffull\else\vspace*{-2mm}\fi
The clause gadget is fairly simple and the construction for the clause 
$c_3 = x_1 \vee x_4 \vee x_5$ is shown in Fig.~\ref{fig:gadgets}(b). If 
any of the lime objects are removed, then the blue object will have an 
exposed thin handle (marked by the red arrows) that can be used for 
lifting and removing the (blue) clause object. A clause gadget will have 
an {\em extension} piece for a connecting variable if the rectangle (e.g., 
the lime piece) connecting the clause gadget and the variable gadget is 
the lowest one on the variable gadget side. In this example, the clause 
gadget for $c_3$ has extensions for $x_1$ and $x_5$ (the long horizontal 
blue extrusions). Comparing with Fig.~\ref{fig:mpsat}, if a connection 
between a clause and a variable is the leftmost one for the variable, 
then there is an extension piece for the gadget for that connection. 

\iffull\else\vspace*{-3mm}\fi
\subsection{Reducing \mpsat to Optimal Clutter Removal}
\iffull\else\vspace*{-2mm}\fi
The complete \crp instance constructed from the \mpsat instance is given 
in Fig.~\ref{fig:reduction}, which is a straightforward assembly of the 
variable and clause gadgets. The additional items are: {\em (i)} two 
extra gray objects at the bottom that can only be lifted and removed 
after the lowest placed positive and negative clause gadgets are 
removed, {\em (ii)} the black ``cap'' object $o_{cap}$ that surrounds 
all other objects, and {\em (iii)} three exits (marked with red hexagons). 
$\W$ is not shown but can be understood as the region occupied by the 
construction with some padded space between the construction and $\partial \W$. 
Object $o_{cap}$ isolates all other objects from left and 
right exits. The robot starts at the middle exit. The instance is 
not drawn to scale. The important dimensions are $w_1$ and $w_2$ as 
marked. The distance $2w_1$ is the horizontal span of the two symmetric 
objects in a variable object. $w_1 + w_2$ is the horizontal distance 
from the middle exit to a vertical segment of a clause gadget. We 
assume that all other distances are small when compared with $w_1$ and 
$w_2$, including the vertical span of the instance and all other 
horizontal distances. Vertical span being minimal means all objects are 
relatively ``long and thin''. Given the assumption, the distance between 
the middle exit and other exits is  $w_1 + w_2$. Moreover, in 
Fig.~\ref{fig:reduction}, the lifting points within a dotted rectangle 
are $\varepsilon$-close to the corresponding exit in the same (dotted) 
rectangle with $\varepsilon$ being very small when compare with $w_1$ 
or $w_2$. Lastly, for an \mpsat instance with $n$ variables, the 
construction ensures $w_2 \gg nw_1$. 

The \crp instance is feasible: all green and orange objects can be 
removed first, exposing the connecting rectangular objects, which can 
be subsequently removed. The purple objects and the black object 
can also then be removed. Afterward, the clause objects can be 
lifted and removed. Lastly, the gray objects at the bottom can be 
removed. On the other hand, computing an optimal solution for the 
problem is hard. First, we establish that the \crp instance requires
at least a travel cost of  $(2n+4)w_1+4w_2$.

\begin{figure}[ht!]
\iffull\else\vspace*{-4mm}\fi
\begin{center}
\begin{overpic}[width=5in,tics=5]{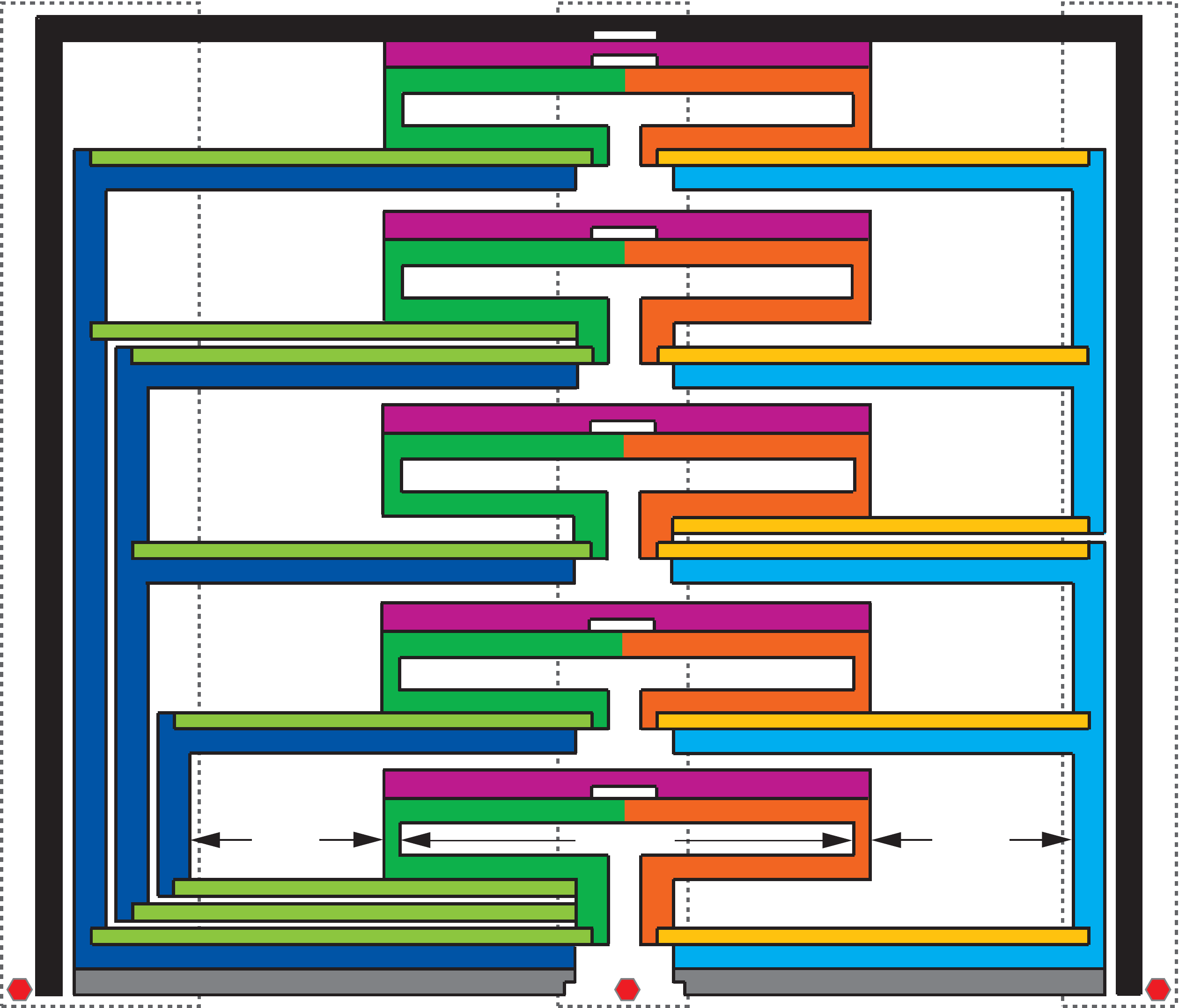}
\put(50,13.3){{\small $2w_1$}}
\put(22,13.3){{\small $w_2$}}
\put(80,13.3){{\small $w_2$}}
\end{overpic}
\end{center}
\iffull\else\vspace*{-4mm}\fi
\caption{\label{fig:reduction} The \crp instance from the \mpsat instance. 
In the middle are the five variable gadgets for $x_1$--$x_5$, from bottom 
to top. On the left side are the three (blue) positive clause gadgets for 
$c_1$--$c_3$, from right to left. On the right side are the two (cyan) 
negative clause gadgets $c_4$ (lower) and $c_5$ (upper). The three red 
hexagons mark the three exits. The robot is initially located at the 
middle exit. The figure is not drawn to scale.}
\iffull\else\vspace*{-8mm}\fi
\end{figure} 

\begin{lemma}\label{l:ocr-cost}
The \crp instance requires a minimum possible cost of $(2n+4)w_1+4w_2$.
\iffull\else\vspace*{-1mm}\fi
\end{lemma}
\begin{proof}
First, given a feasible assignment to the \mpsat instance, we show that 
the \crp instance admits a solution with a total travel distance of 
$(2n+4)w_1+4w_2$. Starting in the middle, if a variable $x_i$ is assigned 
to be positive (resp., negative), the robot removes the green (resp., 
orange) object from the $i$-th variable gadget from the bottom, which 
incurs a distance cost of $2w_1$ per variable. Then, the associated 
purple object and lime (resp., yellow) objects can be removed, which 
incurs minimum extra cost. The step ends with the removal of the black 
cap object. The total distance cost so far is $2nw_1$. 
At this point, every clause gadget object has at least one connecting 
rectangular (lime or yellow) object removed, allowing the clause object 
to be lifted. The robot then moves to leftmost and clears all blue 
(positive) clause gadget objects through the leftmost exit, followed by 
clearing all leftover connecting lime objects. The same is then performed 
on the right side. Lastly, the robot moves back to the middle to clear 
everything else. The main cost in this step is incurred by the travel 
from the middle to the left, then the right, then back, totaling 
$4(w_1 + w_2)$. The grand total is $(2n+4)w_1+4w_2$.

Next, we show the cost is minimal. Beside the above stated removal 
sequence, there are two alternatives. A first is to start with removing 
the black cap object from the top (e.g., through the leftmost exit to
the top of the black object) . This incurs a cost of $3(w_1 + w_2)$ and 
the robot will be at either the leftmost exit or the rightmost exit after 
the removal. Suppose without loss of generality it is the leftmost exit. 
Then, to remove the cyan clause objects (which can only be removed from 
the rightmost) and subsequently the right gray object, the robot must 
travel another $3(w_1 + w_2)$ distance. The total by now is already 
$6(w_1 + w_2) > (2n+4)w_1+4w_2$ because $w_2 \gg nw_1$. 

The second alternative is to remove some of the green/orange/purple 
objects but does not go all the way to the black object before moving to the 
leftmost or rightmost to work on clause objects. Since the two gray 
objects can only be removed from the middle after corresponding clause 
objects are removed, this means that if the robot goes to, e.g., the 
leftmost and then rightmost and then come back to the center, it will 
already incur a cost of $4(w_1 + w_2)$, which means the robot can only 
go to the leftmost and the rightmost once each. This is however 
insufficient because before the black cap object is removed, a blue 
or cyan clause object can only be removed from the center exit and
therefore, multiple trips to the leftmost or the rightmost are necessary. 
~\qed
\end{proof}

\begin{theorem}\label{t:ocr-np-hard}
\iffull\else\vspace*{-2mm}\fi
Planar optimal clutter removal is NP-hard. 
\iffull\else\vspace*{-2mm}\fi
\end{theorem}
\begin{proof}
The proof of Lemma~\ref{l:ocr-cost} already shows that a solution to the 
\mpsat instance leads to a \crp solution of cost $(2n+4)w_1+4w_2$; we only 
need to prove the other direction. Assume the constructed \crp problem has 
an optimal solution with a total cost of $(2n+4)w_1+4w_2$. As has been 
established, the robot must travel to the leftmost and rightmost side at
most once and then eventually return to the middle (with costs 
$4(w_1 + w_2)$), suggesting that the black cap object must be removed 
first before any clause objects can be removed. To be able to remove the 
black cap object, at least $n$ variable objects must be removed, which 
incur a cost of $2nw_1$. Because this already exhausted the total cost, 
No more than $n$ variable objects can be removed before the black cap 
object is. Afterward, we may assume without loss of generality 
that the robot moves to the leftmost to remove the blue clause objects 
through the left exit and must remove all of them before traveling back. 
This implies that the removed green variable objects 
``satisfies'' all the blue (positive) clause objects. Similarly, the 
cyan (negative) clause objects must also be ``satisfied'' by the removed 
orange variable objects. This yields a satisfactory assignment for
the \mpsat problem. ~\qed
\end{proof}

Since it is easy to verify whether a given solution is optimal or not, 
planar \crp (the simplified combinatorial version without considering 
complex motion planning) is also in NP. Therefore, this version of 
optimal clutter removal is NP-complete. A corollary follows that 
applies to two exits. 

\begin{corollary}
\iffull\else\vspace*{-2mm}\fi
Planar optimal clutter removal is NP-hard with two exits. 
\iffull\else\vspace*{-2mm}\fi
\end{corollary}
\begin{proof} We note that the \crp instance can be ``bent'' in the 
middle with the left and right sides bending up until they almost meet. 
Since the initial vertical span of the \crp instance is negligible, 
this causes the two exits to also be $\varepsilon$-close, i.e., they can 
be merged into a single exit. This yields a new \crp instance with two 
exits. The NP-hardness proof continues to work with the updated optimal 
cost being no more than $(2n+2)w_1 + 2w_2$.~\qed
\end{proof}

\remark Though the robot can lift large objects by assumption, lifting a 
long object in the middle and then going through an exit can potentially 
lead to issues. In the constructed \crp instance, these are the purple 
objects and the black cap object; all other objects are lifted from one
end (recall that the vertical span of the construction is negligible). 
The issue can be resolved by breaking an involved object into two equal 
pieces in the middle which can then be taken away separately (see 
Fig.~\ref{fig:break-apart}), without incurring much additional travel 
cost (the vertical span of the black object is assumed to be very small). 
By doing this, the robot will always be holding a long object from one end. 

\begin{figure}[ht!]
\iffull\else\vspace*{-5mm}\fi
\begin{center}
\begin{overpic}[width=5in,tics=5]{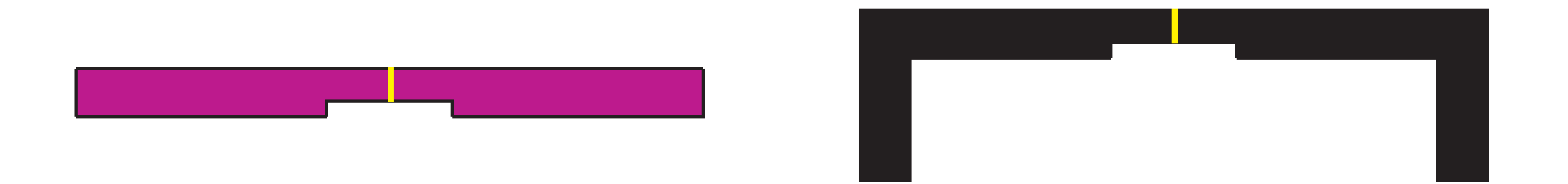}
\end{overpic}
\end{center}
\iffull\else\vspace*{-7mm}\fi
\caption{\label{fig:break-apart} How the purple and black objects can be 
broken into more manageable pieces, as indicated by the yellow lines, 
without affecting solution structure. Objects are not drawn to scale.}
\iffull\else\vspace*{-6mm}\fi
\end{figure}

\iffull\else\vspace*{-3mm}\fi
\section{Resolution-Complete Algorithms for \crp}\label{section:algorithm}
\iffull\else\vspace*{-3mm}\fi
By Proposition~\ref{p:complete}, applying a resolution-complete 
algorithm for identifying candidate objects for removal
will result in a resolution-complete algorithm for \crp. 
In this section, we will construct several resolution-complete 
algorithms realizing varying levels of optimality guarantees. 
With the preparation done in Section~\ref{sec:alg-comp}, our 
construction assumes knowledge of currently graspable objects and 
the costs of reaching them, and focuses on the selection of 
clutter removal sequence based on these information. 

Before introducing the algorithms, we present an example that 
illustrates additional structures of \crp in Fig.~\ref{fig:optimal}(a), 
for which a greedy removal sequence is indicated in Fig.~\ref{fig:optimal}(b) 
and an optimal one is given in Fig.~\ref{fig:optimal}(c). Assuming 
similar grasping cost, the travel time used by the greedy approach is 
about 1.3 times of that used by the optimal sequence. On the other hand, 
Theorem~\ref{t:ocr-np-hard} indicates that optimal removal sequences can 
be hard to come by. This motivates the construction of both optimal and 
greedy algorithms for \crp. 

\begin{figure}[ht!]
\iffull\else\vspace*{-4mm}\fi
\begin{center}
\begin{overpic}[width=4.5in,tics=5]{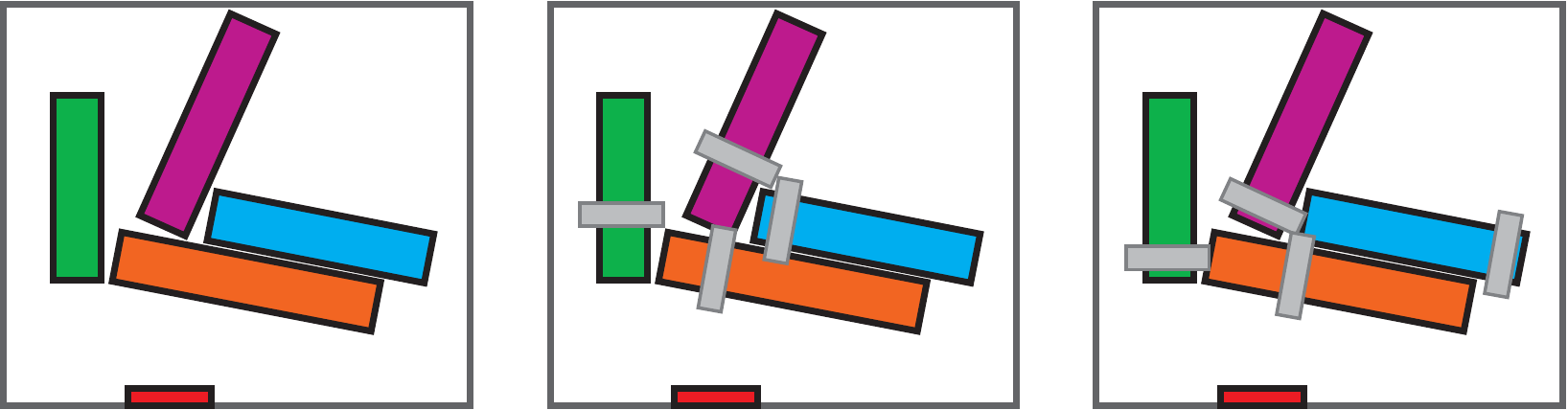}
\put(35,11){{\small $1$}}
\put(42,16){{\small $2$}}
\put(44.5,3){{\small $3$}}
\put(52,14){{\small $4$}}
\put(95,14){{\small $1$}}
\put(81.5,2.5){{\small $2$}}
\put(74,4){{\small $3$}}
\put(77,15){{\small $4$}}
\put(14,-4){{\small (a)}}
\put(49,-4){{\small (b)}}
\put(83,-4){{\small (c)}}
\end{overpic}
\end{center}
\iffull\else\vspace*{-4mm}\fi
\caption{\label{fig:optimal} (a) A \crp instance. (b) A greedy removal 
sequence, with the gray bar marking the grasping position by a 2-finger
gripper (top view). (c) An optimal removal sequence.}
\iffull\else\vspace*{-6mm}\fi
\end{figure}
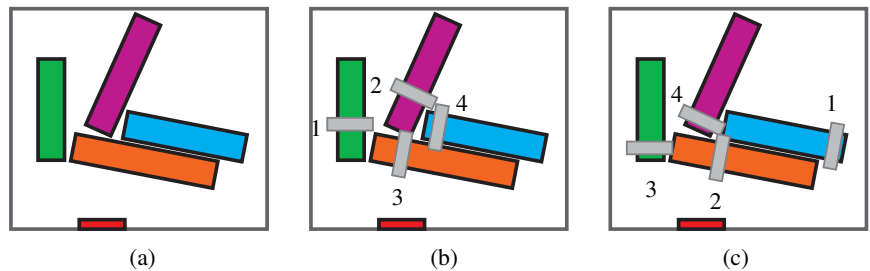 

\iffull\else\vspace*{-4mm}\fi
\subsection{Single Exit: Exhaustive Search with Dynamic Programming}
\iffull\else\vspace*{-2mm}\fi
An exhaustive search approach based on backtracking \cite{van2009path} may 
be applied to derive an algorithm to search for the optimal object removal 
sequence. The basic idea is straightforward: all possible object removal 
sequences are examined and the one with the best cost is chosen. In the 
context of the current study, a search tree is grown and explored in a depth 
first manner, with each path from the root to a leaf node representing a 
complete object removal sequence. Since all permutations are examined, 
this guarantees an optimal solution is found as long as the cost estimate 
for grasping and transporting each object is accurate. This later part is in
turn guaranteed in a resolution-complete manner in this work, because
resolution-complete algorithms are used to build the common motion and 
grasp planning components. Complete exhaustive search is also possible. 

A daunting challenge in examining all branches of a search tree with 
depth $n$ is the nominal time complexity of $O(n!)$, prohibitively 
expensive for even small $n$ (e.g., $n > 5$). For the \crp problem, 
a form of {\em dynamic programming} (DP) may be applied to significantly 
reduce this complexity as follows. Let $I = \{1, \ldots, n\}$
and $I' \subset I$. Let $J(I')$ denote the optimal cost of removing 
all objects with indices in $I'$ assuming that objects with indices in 
$I\backslash I'$ are already removed. Then we have the key DP recursion 
\begin{align}\label{eq:recursion}
\iffull\else\vspace*{-1mm}\fi
J(I') = \min_{i \in I'} \{ c_i + J(I'\backslash\{i\})\},
\iffull\else\vspace*{-1mm}\fi\end{align}
where $c_i$ is the cost of removing $o_i$ assuming that objects with 
indices in $I\backslash I'$ are already removed. For $|I'| = k$,
there are $\binom{n}{k}$ possible $I'$ and for each, computing 
\eqref{eq:recursion} requires a cost of $O(n)$. This then yields a total
computational cost of 
\[
\iffull\else\vspace*{-1mm}\fi
O(n) \big[\binom{n}{0} + \binom{n}{1} + \ldots + \binom{n}{n-1}\big]= O(n2^n). 
\]
We note that $2^n$ grows much slower than $n! \sim \sqrt{2\pi n}(n/e)^n$. 

In addition to DP, two additional structural 
properties of \crp can be exploited to further boost computational efficiency 
without affecting solution optimality, namely: 
\begin{itemize}[leftmargin=.16in]
\iffull\else\vspace*{-1mm}\fi
\item \textbf{Reachability}. Objects in clutter naturally create 
workspace obstacles, limiting the access to other objects and thus 
reducing effective search branching factor. For example, the purple 
object in Fig.~\ref{fig:optimal}(a) may be inaccessible initially.
\item \textbf{Object clustering}. It is possible that the objects form 
{\em clusters} that are independent in terms of the removal cost, i.e., 
there may be two or more isolated ``piles'' of objects. Note that in
some cases, a pile may need to be removed first before another can be 
removed effectively. Clusters can be readily identified by grouping 
objects that are close to each other. 
\iffull\else\vspace*{-3mm}\fi
\end{itemize}

\iffull\else\vspace*{-2mm}\fi
\subsection{Single Exit: Greedy Best-First Search}
\iffull\else\vspace*{-1mm}\fi
The NP-Hardness of optimally solving \crp means that exhaustive search 
cannot run in polynomial time, which prompts the development of greedy 
approaches: the object with the lowest {\em local} removal cost is selected 
and removed; the same process is then recursively applied until all objects 
are cleared. In addition to the basic greedy best-first strategy which 
only looks at a single step, two more involved methods are also explored:
\begin{itemize}[leftmargin=.16in]
\item \textbf{Multi-step best-first search}. This method computes cost 
after growing the search tree to some depth $k \ge 1$. The approach, 
a finite-horizon technique, balances between increased computation and 
better solution optimality. For example, if $k$ is set to $3$, then the case 
from Fig.~\ref{fig:optimal} can be solved optimally using multi-step 
best-first search. 
\item \textbf{Monte Carlo Tree Search (MCTS)}. As the core complexity 
arises from finding a best path along a search tree, another natural 
choice is Monte Carlo Tree Search (MCTS) \cite{coulom2006efficient,
kocsis2006bandit}, which performs limited search tree exploration with 
varying depth along different tree branches. This can be viewed as a 
Monte Carlo version of the multi-step best-first search strategy. 
\end{itemize}

In terms of computational complexity, all greedy approaches described here
have low polynomial dependency on $n$, the number of objects. 
\jy{Some practical considerations are commented out here; maybe add these
back later.}

\iffull\else\vspace*{-4mm}\fi
\subsection{Multiple Exits: Extending Exhaustive Search with Dynamic Programming }
\iffull\else\vspace*{-2mm}\fi
Algorithms for the single-exit case generalize to multiple exits. Whereas 
the greedy algorithm requires little change, exhaustive search with DP 
requires a non-trivial extension. Let $I = \{1, \ldots, n\}$ and $I' 
\subset I$. Let $J_{ij}(I')$ denote the optimal cost of removing 
all objects with indices in $I'$ with the robot starting from exit 
$i$, and ending at exit $j$, assuming that objects with indices in 
$I\backslash I'$ are already removed. Let $E$ denote the set of all 
exits of the environment. The updated DP recursion is 
\begin{align}\label{eq:recursion-m}
J_{ij}(I') = \min_{e \in E}{\min_{k \in I'} \{ c_{ej}(k) + J_{ie}(I'\backslash\{k\})\}}, 
\end{align}
where $c_{ij}(k)$ is the cost of starting at exit $i$ and removing $o_k$ 
from through exit $j$, assuming that objects with indices in $I\backslash 
I'$ are already removed. For $|I'| = k$, if we assume the robot always 
start from a fixed exit, there are $|E|\binom{n}{k}$ possible $I'$ and 
for each, computing \eqref{eq:recursion-m} costs $O(|E|n)$. The total is 
\[
O(|E|n) \big[|E|\binom{n}{0} + |E|\binom{n}{1} + \ldots + |E|\binom{n}{n-1}\big]= O(n2^n|E|^2). 
\]

\subsection{Multiple Exits: Voronoi Partitions}
\iffull\else\vspace*{-1mm}\fi
\begin{wrapfigure}[8]{r}{2.05in}
\iffull\vspace*{-8mm}\else\vspace*{-7mm}\fi
\begin{overpic}[width=2in,tics=5]{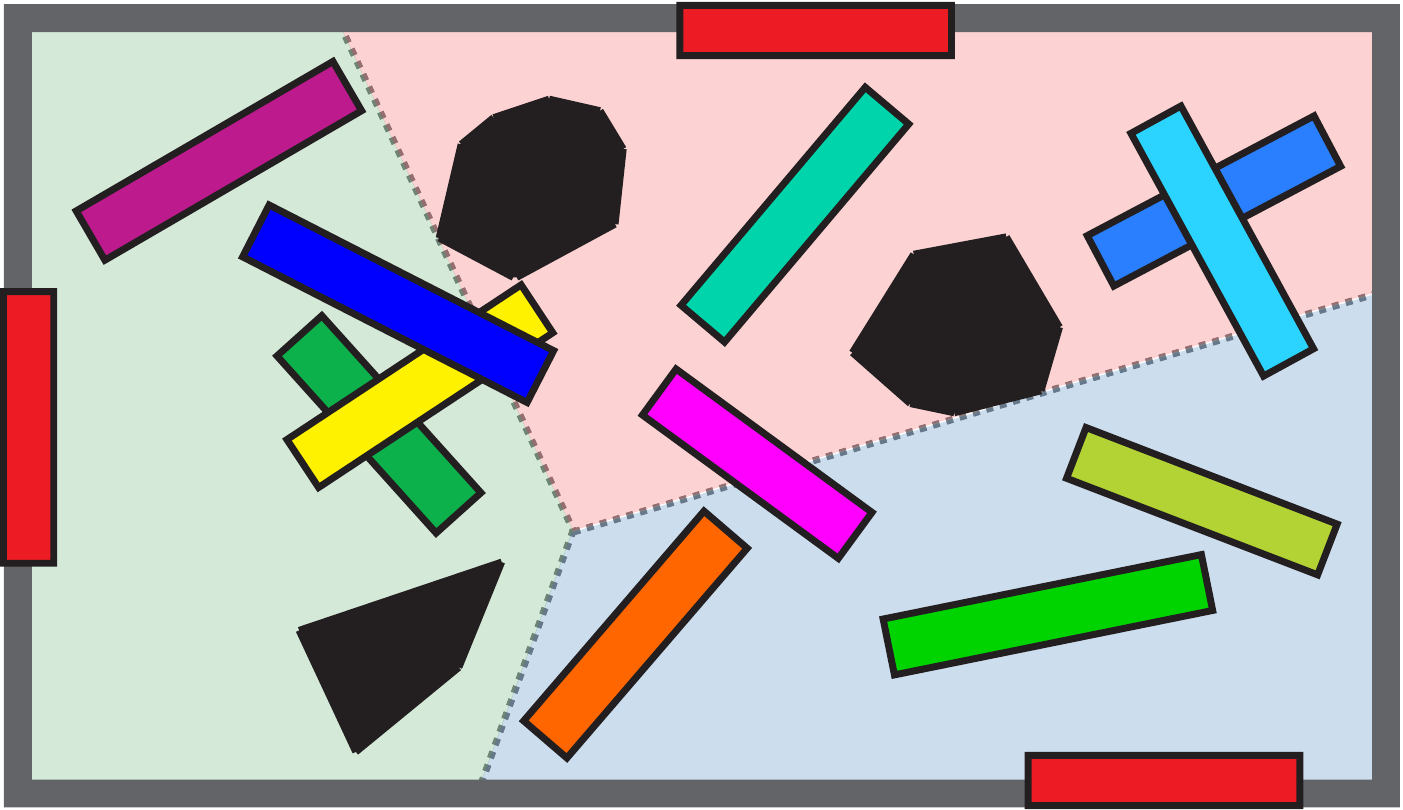}
\end{overpic}
\iffull\vspace*{-4mm}\else\vspace*{-4mm}\fi
\caption{\label{fig:cr2dv} The Voronoi partition of the example \crp  
(Fig.~\ref{fig:cr2d}).}
\iffull\vspace*{-3mm}\else\vspace*{-3mm}\fi
\end{wrapfigure}
An alternative algorithm for multiple exits may look at the Voronoi 
partitions induced by $\W$ and the exits, and let the mobile robot remove 
objects through an exit if the object falls within the corresponding 
Voronoi region. After the robot finishes working with objects within a 
Voronoi partition, it moves to the next Voronoi partition. As an example, 
the Voronoi regions for the three-exit scenario in Fig.~\ref{fig:cr2d} 
is given Fig.~\ref{fig:cr2dv}. Through each exit, the robot will remove 
around four objects. 

When objects are not heavily entangled, using Voronoi partition plus 
any single-exit method incurs an additional travel cost equaling  
at most the length of $\partial \W$. Indeed, simulation study shows that an Voronoi 
based algorithm performs quite well when compared with direct extensions of 
single-exit methods. Moreover, Voronoi partitions can be readily obtained
based on the shape of $\W$ and the exits' locations, whereas direct extensions
of single-exit methods may require additional sensing information to work. 
That is, to make estimates on which objects can be removed and the associated 
costs, the robot may need to physically travel through each exit to acquire 
that information. 

\iffull\else\vspace*{-3mm}\fi
\section{Experimental Evaluation}\label{section:evaluation}
\iffull\else\vspace*{-3mm}\fi
A sequence of experiments were designed to evaluate the effectiveness 
of the algorithms for \crp. Each experiment also provides
additional new insights into the structure of \crp. The algorithms were 
implemented in C++ and  executed on a quad-core Intel CPU at 3.3GHz with 
32GB RAM. A video of a simulated Kuka youBot carrying out \crp tasks is 
provided that corroborates the evaluation described in this section. 

\iffull\else\pagebreak\fi
\subsection{Single-Exit Scenarios}
\begin{wrapfigure}[6]{r}{1.45in}
\iffull\vspace*{-8mm}\else\vspace*{-8mm}\fi
\begin{overpic}[width=1.4in,tics=5]{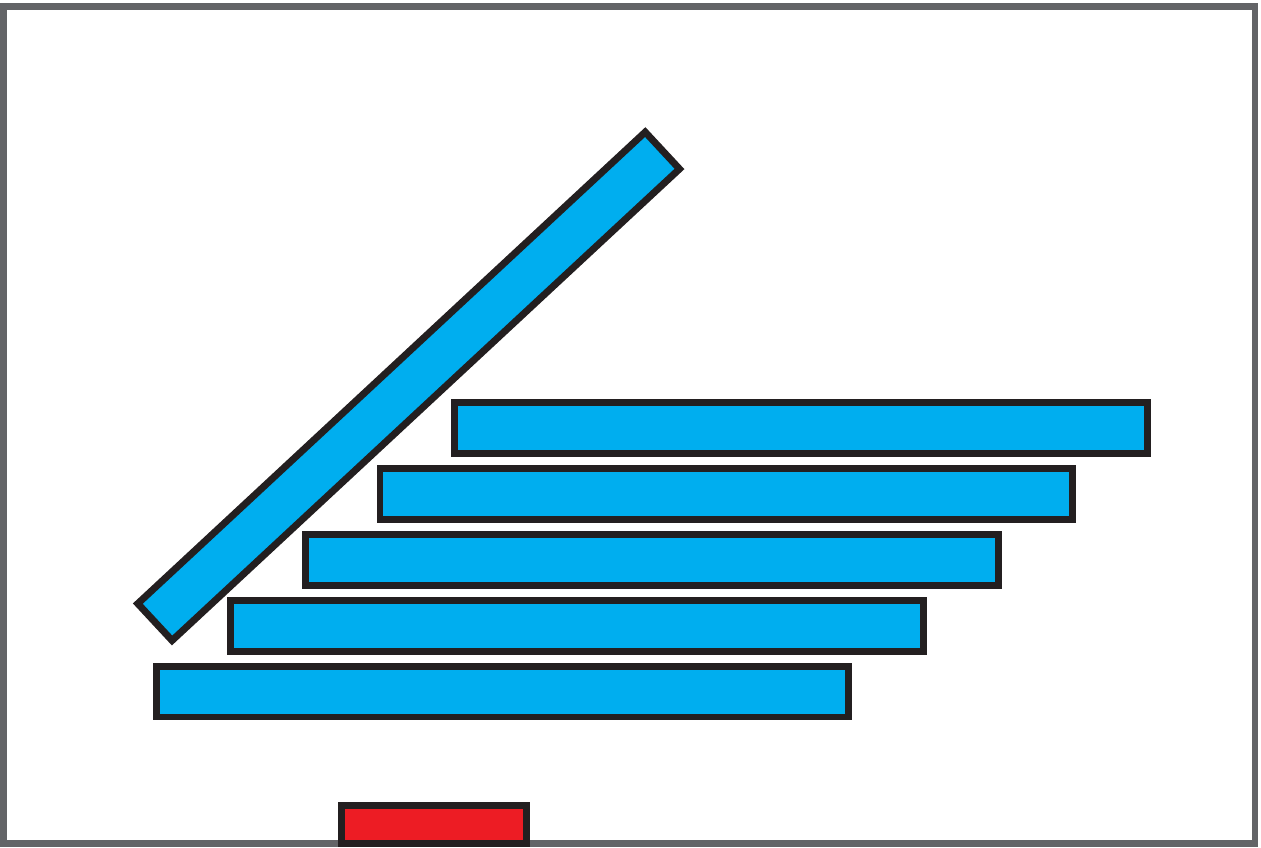}
\end{overpic}
\iffull\else\vspace*{-6mm}\fi
\end{wrapfigure}
\iffull\else\vspace*{-1mm}\fi
There are cases (e.g., Fig.~\ref{fig:optimal}) where a greedy algorithm 
for \crp can be rather sub-optimal when compared with the exhaustive 
algorithm. One may ask the natural question of whether such differences 
actually matter in practice. To evaluate this, we integrated the full 
solution pipeline with Gazebo using KUKA youBot as the mobile robot (see, 
e.g., Fig.~\ref{fig:scene}). As a first experiment, we evaluated the 
execution time of plans obtained by both the exhaustive and greedy 
algorithm for the scene illustrated on the right. While the 
plan provided by the greedy algorithm has a total travel distance 
that is $2.1$ times that from the exhaustive one, the ratio of
execution time in Gazebo for the two cases is about $1.4$ (see 
the submitted video). The difference in the two ratios ($2.1$ v.s. 
$1.4$) is due to the time required for grasping/releasing
the objects, which is almost the same for both. The example can be 
readily generalized to yield a family of ``bad cases'' by stacking
the same pattern over and over. 

\crp scenarios such as these illustrated in Fig.~\ref{fig:reduction}, 
Fig.~\ref{fig:optimal}, and the previous example are highly non-random. 
It is unlikely to encounter these in practice, where clutter
tends to have a more random placement. Our second experiment focuses
on different random single-exit clutter removal scenarios with the following 
possible opposing properties:
\begin{itemize}[leftmargin=.38in]
\iffull\else\vspace*{-1mm}\fi
\item[\textbf{S/C}:] Whether the objects are {\em scattered} uniformly in the 
room or {\em centered} in the room. 
\item[\textbf{R/A}:] Whether the objects are oriented in {\em random directions} or 
they are {\em axis-aligned}.
\item[\textbf{O/N}:] Whether the objects are {\em overlapping} or not. 
\iffull\else\vspace*{-1mm}\fi\end{itemize}
Among the eight combinations that were attempted, we select four 
representative settings as illustrated
\begin{figure}[ht!]
    \iffull\else\vspace*{-4mm}\fi
    \begin{center}
    \begin{overpic}[width=6.5in,tics=5]{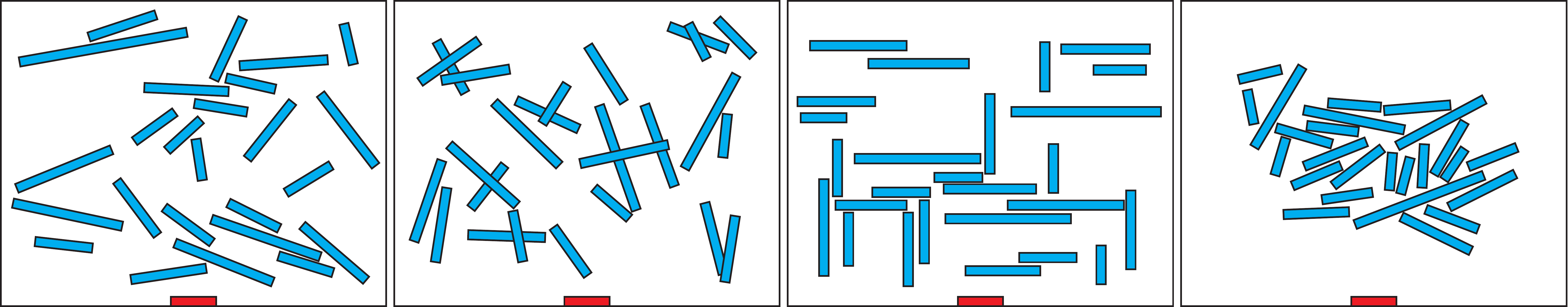}
    \put(10.5,-4){{\small (a)}}
    \put(35.5,-4){{\small (b)}}
    \put(60.5,-4){{\small (c)}}
    \put(85.5,-4){{\small (d)}}
    \end{overpic}
    \end{center}
    \iffull\else\vspace*{-2mm}\fi
    \caption{\label{fig:four-cases} Four selected settings: (a) \textbf{SRN}.
		(b) \textbf{SRO}. (c) \textbf{SAN}. (d) \textbf{CRN}. }
    \iffull\else\vspace*{-4mm}\fi
\end{figure}
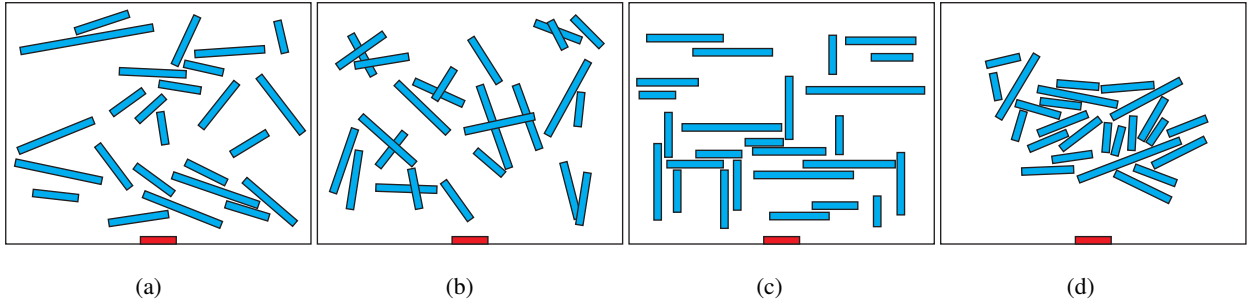 
in Fig.~\ref{fig:four-cases}. The objective of the second experiment is 
to evaluate the relative computational complexity of different cluttered 
scenes. For this, both exhaustive search and greedy search are attempted, 
with the results for exhaustive search (with heuristics) shown in 
Fig.~\ref{fig:four-cases-ex}. 
We note that all test cases are generated randomly with varying object 
numbers and lengths; for each setting and each number of objects 
($5$-$40$), $20$ test cases are created. A data point corresponds to 
the average over the $20$ cases. For each case, a time limit of $400$
seconds is placed. If one of the $20$ cases exceeds the limit, no data point 
for that setting is included. The travel cost is unit-less.
\begin{figure}[ht!]
    \begin{center}
    \begin{overpic}[width=3in,tics=5]{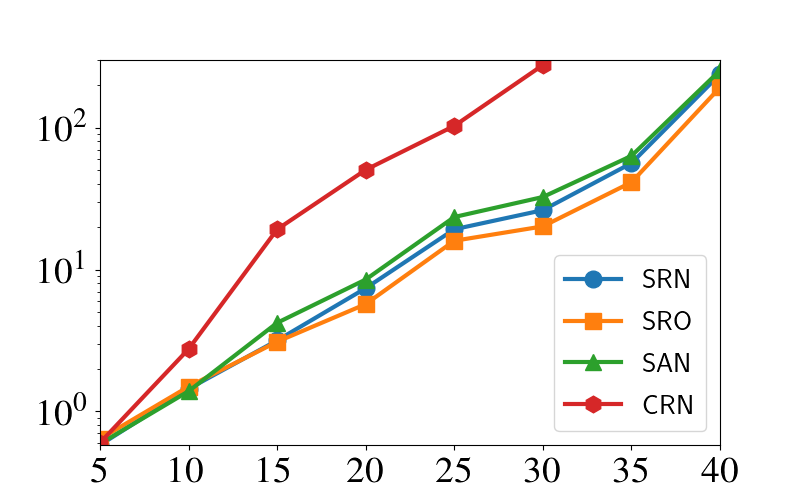}
    \put(-8,10){\rotatebox{90}{\scriptsize Computation Time}}
    \end{overpic}
		\hspace{5mm}
    \begin{overpic}[width=3in,tics=5]{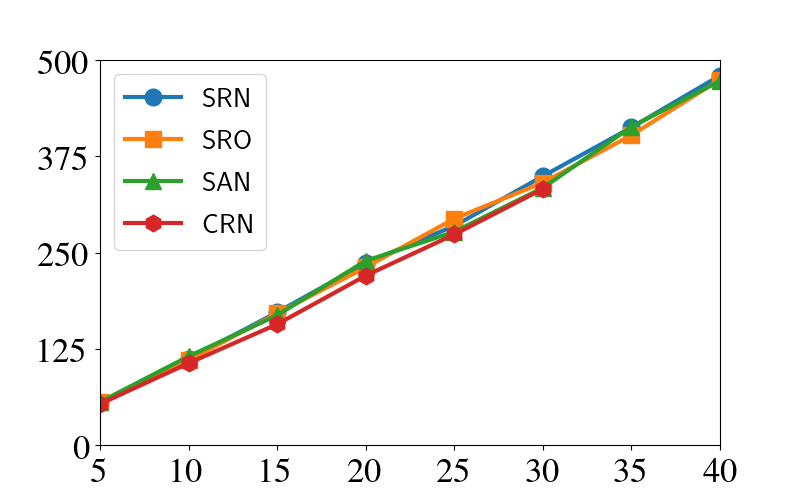}
    \put(-6,20){\rotatebox{90}{\scriptsize Total Cost}}
    \end{overpic}

    \end{center}
    \iffull\else\vspace*{-7mm}\fi
    \caption{\label{fig:four-cases-ex} The computational time and total 
		distance cost using Exhaustive search (with heuristics) on 
		four different settings as illustrated in Fig.~\ref{fig:four-cases}. 
		The $x$-axis legends show the number of robots.}
		\iffull\else\vspace*{-6mm}\fi
\end{figure} 
From the result the following observations can be made:

\begin{itemize}[leftmargin=.16in]
\iffull\else\vspace*{-1mm}\fi
\item Scenes with overlapping objects (\textbf{SRO}, 
Fig.~\ref{fig:four-cases}(b)) are easier than scenes 
with non-overlapping objects (\textbf{SRN}, Fig.~\ref{fig:four-cases}(a)). 
Axis-aligned cases (\textbf{SAN}, Fig.~\ref{fig:four-cases}(c)) 
are slightly harder than cases where objects' orientations 
are more random (\textbf{SRN}, \textbf{SRO}). 
\item Centered cases (\textbf{CRN}, Fig.~\ref{fig:four-cases}(d)) are 
much more challenging (notice the logarithmic scale computation time in 
Fig.~\ref{fig:four-cases-ex}). This is due to 
two reasons: the objects are closer and more objects are graspable, 
making the branching factor larger in the search tree. 
\end{itemize}
\iffull\else\vspace*{-1mm}\fi

We further observe that (from data omitted due to space constraint) 
somewhat surprisingly, the greedy algorithm computes solutions 
for all cases with nearly the same total distance costs and does so 
with much less computation time. To study this further, we fixate 
on \textbf{SRN} as we expect this to be typical and also harder 
than \textbf{SRO}. Multiple algorithms were tested and the result is 
given in Fig.~\ref{fig:algorithms}, which clearly shows that the basic 
greedy approach works quite well in terms of optimality and runs much 
faster than other methods as the number of objects in clutter increases. 
Other greedy methods (multi-step, MCTS) take more time but also produce 
slightly more optimal solutions. 
\begin{figure}[ht!]
    \iffull\else\vspace*{-7mm}\fi
    \begin{center}
    \begin{overpic}[width=3in,tics=5]{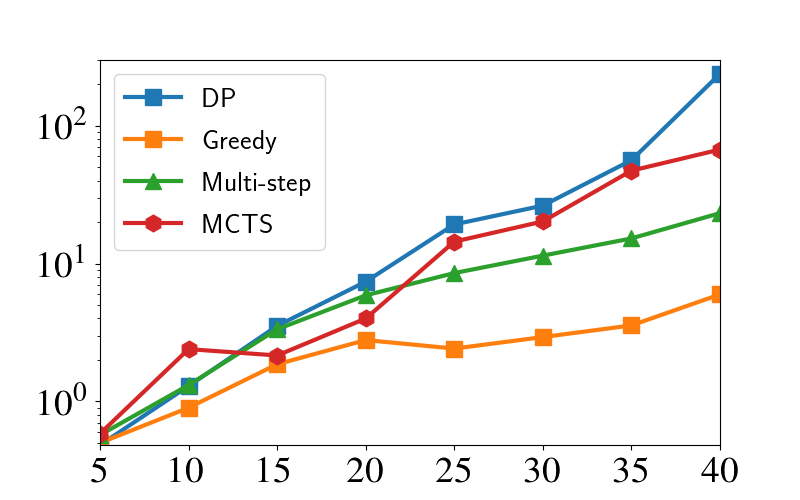}
    \put(-8,10){\rotatebox{90}{\scriptsize Computation Time}}
    \end{overpic}
		\hspace{5mm}
    \begin{overpic}[width=3in,tics=5]{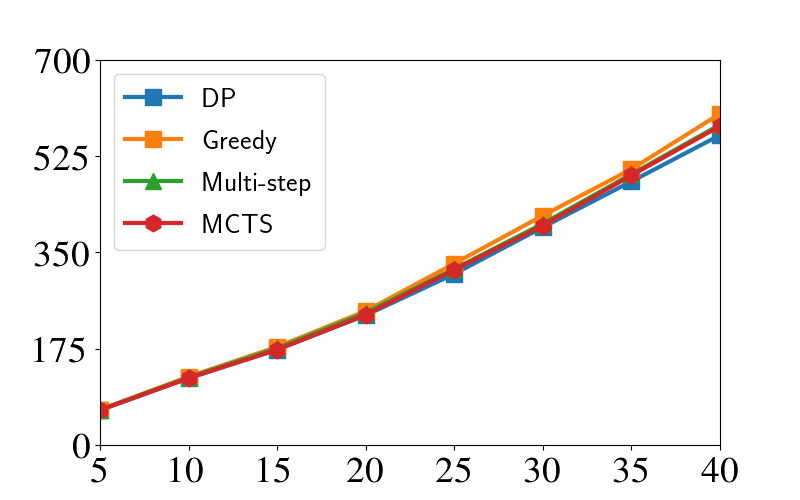}
    \put(-6,20){\rotatebox{90}{\scriptsize Total Cost}}
    \end{overpic}

    \end{center}
    \iffull\else\vspace*{-8mm}\fi
    \caption{\label{fig:algorithms} Computation times and total travel costs
    from multiple algorithms. 
		}
		\iffull\else\vspace*{-6mm}\fi
\end{figure}

\begin{wrapfigure}[6]{r}{2in}
\iffull\vspace*{-8mm}\else\vspace*{-8mm}\fi
\begin{overpic}[width=1.95in,tics=5]{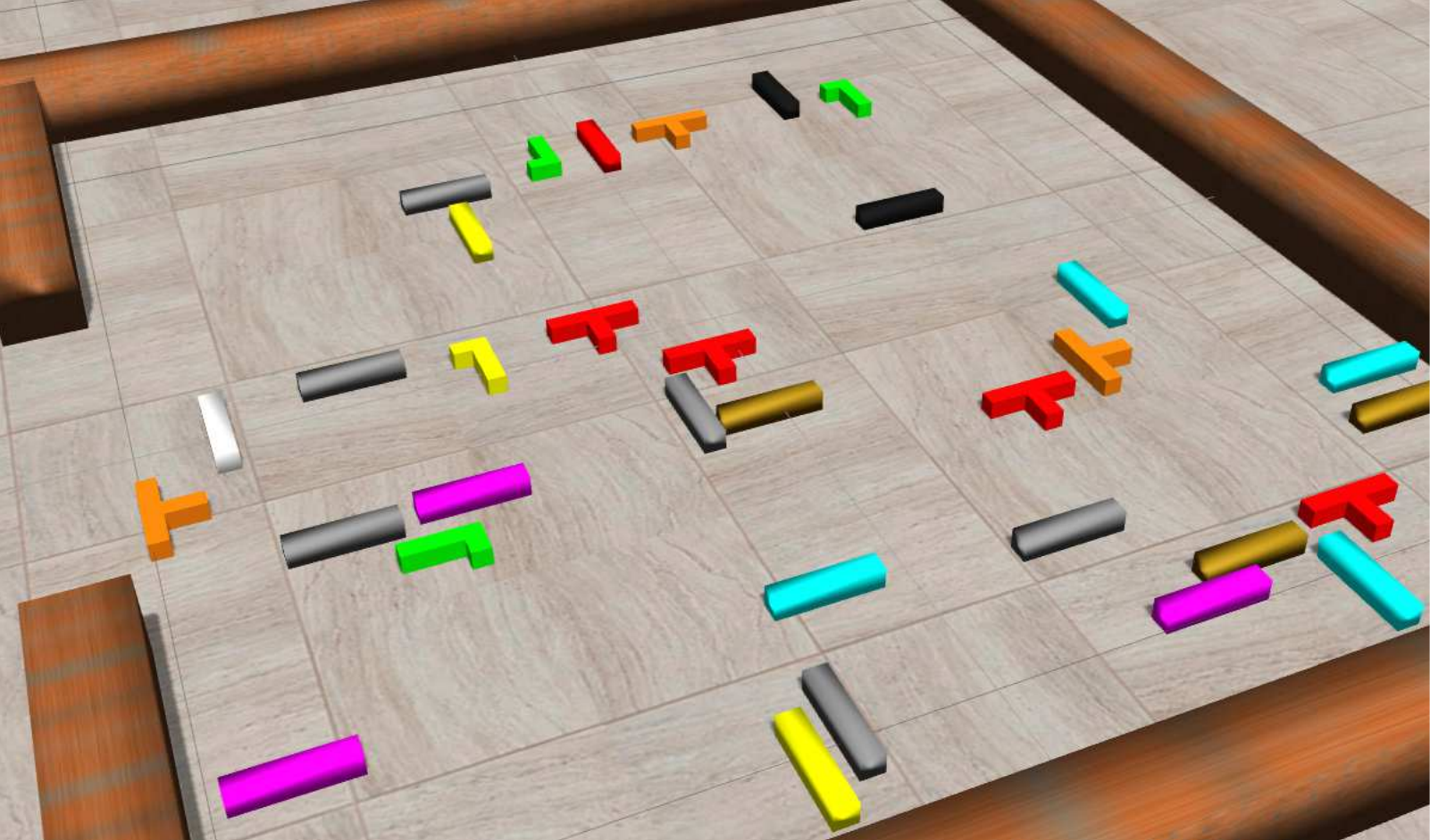}
\end{overpic}
\iffull\else\vspace*{-6mm}\fi
\end{wrapfigure}
Experiments were also carried out to evaluate the effect of two additional 
factors: obstacles and more complex shapes. A typical test case of 
the former is shown in Fig.~\ref{fig:scene} and a typical case for the 
later is illustrated on the right, with Tetris-like objects, for which 
grasp planning becomes more challenging. For both settings, results are
highly similar to what is shown in Fig.~\ref{fig:algorithms}.

\subsection{Multi-Exit Setup}
\iffull\else\vspace*{-1mm}\fi
For the multi-exit setup, we also attempted a number of experiments. 
Given the similarity to the single-exit case, little new insights were 
obtained in running the experiments similar to the single-exit case, 
except that the branching factor becomes larger due to the availability 
of more exits, which makes more objects accessible at once. A new set of 
experiments were also created to evaluate the effectiveness of the 
Voronoi partition based algorithm, which is compared with exhaustive 
search and the basic greedy algorithm without the Voronoi heuristic. The 
test cases are the single-exit \textbf{SRN} ones now with three exits 
selected randomly along the boundary $\partial \W$. The robot may travel 
along the outside of $\partial \W$ between exits. The experimental result 
is plotted in Fig.~\ref{fig:multi-exit}.
\begin{figure}[ht!]
	\begin{center}
		\begin{overpic}[width=3in,tics=5]{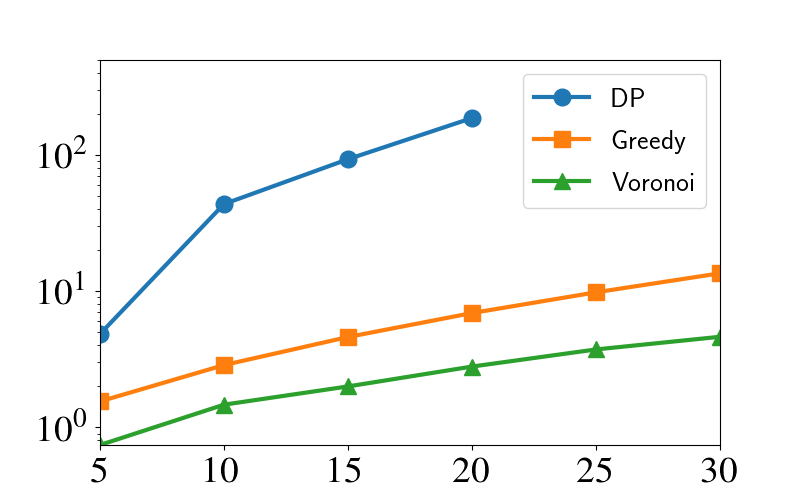}
    \put(-8,10){\rotatebox{90}{\scriptsize Computation Time}}
    \end{overpic}
		\hspace{5mm}
    \begin{overpic}[width=3in,tics=5]{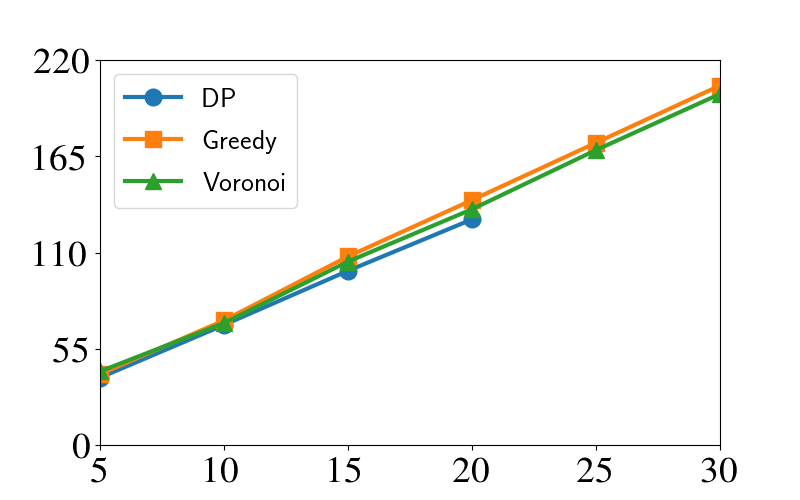}
    \put(-6,20){\rotatebox{90}{\scriptsize Total Cost}}
    \end{overpic}
  \end{center}
  \iffull\else\vspace*{-8mm}\fi
  \caption{\label{fig:multi-exit} Multi-exit experiments running exhaustive 
	with dynamic programming, greedy, and Voronoi partition based greedy algorithms.}
  \iffull\else\vspace*{-6mm}\fi
\end{figure}

Again, we observe that the total costs exhibit little difference among
the methods. However, the Voronoi based method demonstrates superior 
scalability, running much faster than the basic greedy algorithm and exhaustive 
search. At the same time, DP-based exhaustive search can effectively handle 
over 20 objects efficiently and provides slightly better cost than the 
greedy methods. 

\iffull\else\vspace*{-3mm}\fi
\section{Conclusion and Discussion}\label{section:conclusion}
\iffull\else\vspace*{-3mm}\fi
In this paper, we investigate the clutter removal problem (\crp), performing extensive 
structural and algorithmic studies for both single- and multi-exit cases.
After showing that the problem can be NP-hard to optimally solve, we
develop resolution-complete exhaustive search algorithms for \crp. With DP, the
algorithms are effective for both single and multiple exits. We 
also show that typical settings can be efficiently solved using greedy
algorithms, which have even better scalability and produce solutions 
that are fairly close to being optimal. Our algorithms are capable of 
computing high-quality solutions in seconds for scenes with tens of 
objects. A key conclusion from the empirical evaluation is that greedy 
approaches may be applied as a first resort; when there are additional 
computational resources, longer horizons may also be explored using more 
exhaustive approaches to further enhance plan optimality. 

\bibliographystyle{tex/spmpsci}
\iffull\else\begin{footnotesize}\fi
\bibliography{bib/bib,bib/jingjin}
\iffull\else\end{footnotesize}\fi

\iffull
\section*{Appendix: Hardness of Optimal Clutter Removal with a Single Exit}\label{section:complexity-se}
The proof for the single exit case, also done through the reduction 
from \mpsat, is considerably more involved. It requires both a 
careful analysis of the planar structure of \mpsat instances and 
significantly revised gadget design. For introducing the somewhat 
complex proof, we use a modified \mpsat instance illustrated in 
Fig.~\ref{fig:mpsat6}. The instance has variables $x_1$--$x_5$ and 
clauses $c_1 = x_1\vee x_2, c_2 = x_2 \vee x_3 \vee x_4, c_3 = x_1 
\vee x_4 \vee x_5, c_4 = \neg x_1 \vee \neg x_2 \vee \neg x_3$, 
$c_5 = \neg x_3 \vee \neg x_4 \vee \neg x_5$, and $c_6 = \neg x_1 
\vee \neg x_3 \vee \neg x_5$. We assume that we always work with a 
fixed planar embedding of a given \mpsat instance, which can always 
be arranged as illustrated in Fig.~\ref{fig:mpsat6}, where a 
horizontal chain of variables separates the positive clauses and 
negative clauses, with the positive (resp., negative) clauses 
staying above (resp., below) the chain of variables. 
In a general \mpsat instance, we let the number of variables be $n$ and 
the number of clauses be $m$. 
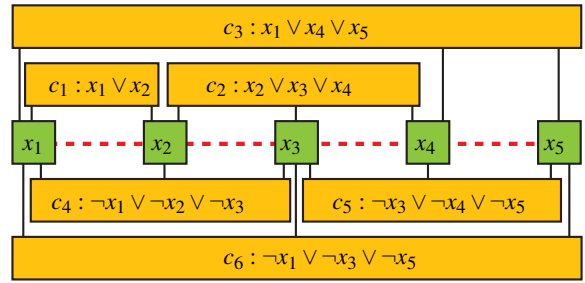
\begin{figure}[ht!]
\vspace*{-2mm}
\begin{center}
\begin{overpic}[width=3in,tics=5]{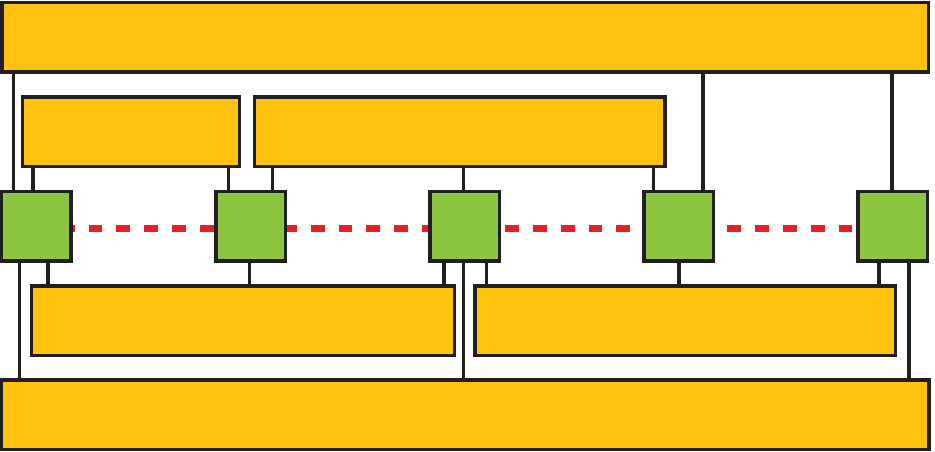}
\put(2,22.8){{\small $x_1$}}
\put(24.5,22.8){{\small $x_2$}}
\put(47,22.8){{\small $x_3$}}
\put(70.5,22.8){{\small $x_4$}}
\put(93,22.8){{\small $x_5$}}
\put(37,43){{\small $c_3: x_1 \vee x_4 \vee x_5$}}
\put(34,33){{\small $c_2: x_2 \vee x_3 \vee x_4$}}
\put(6.5,33){{\small $c_1: x_1 \vee x_2$}}
\put(7,12.5){{\small $c_4:\neg x_1 \vee \neg x_2 \vee \neg x_3$}}
\put(56,12.5){{\small $c_5:\neg x_3 \vee \neg x_4 \vee \neg x_5$}}
\put(37,2.8){{\small $c_6:\neg x_1 \vee \neg x_3 \vee \neg x_5$}}
\end{overpic}
\end{center}
\vspace*{-2mm}
\caption{\label{fig:mpsat6}The planar embedding of an \mpsat instance with 
variables $x_1$--$x_5$ and clauses $c_1$--$c_6$.}
\vspace*{-2mm}
\end{figure} 

\subsection{The Planar Structure of \mpsat Instances}
For the type of embedding given in Fig.~\ref{fig:mpsat6}, we first make 
an observation of \mpsat instances regarding the relative arrangements 
of clauses that share the same literal. Without loss of generality we 
look at positive clauses that share some positive (non-negated) literal 
$x_i$. We say two clauses $c_j^1$ and $c_j^2$, both containing $x_i$, are 
{\em adjacent} if they are not separated by clauses or lines  
(connecting clauses and variables) in the embedding. For example, $c_1$ 
and $c_2$ are adjacent in Fig.~\ref{fig:mpsat6}; so are $c_1$ and $c_3$. 
Clauses $c_4$ and $c_5$ are not adjacent as they are separated by the 
connection between $x_3$ and $c_6$. For adjacent clauses $c_j^1$ and 
$c_j^2$, we say that $c_j^1$ {\em encloses} $c_j^2$ if $c_j^2$ is 
between $c_j^2$ and the variable chain. For example, $c_3$ encloses $c_1$ 
as $c_1$ falls between the chain of variables and $c_3$. 

If $c_j^1$ encloses $c_j^2$ and the connection between $x_i$ and $c_j^1$ 
appears to the right of the connection between $x_i$ and $c_j^2$ in 
the embedding, then we say $c_j^1$ {\em right encloses} $c_j^2$ due to 
$x_i$. We write this relationship as $c_j^1 \sr_{x_i} c_j^2$ 
(alternatively, $c_j^2 \pl_{x_i} c_j^1$), or simply $c_j^1 \sr c_j^2$ 
(alternatively, $c_j^2 \pl c_j^1$) as the variable $x_i$ can be readily 
inferred.The {\em left enclose} relationship is similarly defined with 
relation symbols $\sl$ and $\pr$. As examples, in Fig.~\ref{fig:mpsat6}, 
$c_3 \sl c_1$ (due to $x_1$), $c_3 \sr c_2$ (due to $x_4$), $c_6 \sl_{x_1} 
c_4$, $c_6 \sr_{x_3} c_4$, $c_6 \sl_{x_3} c_5$, and $c_6 \sr_{x_5} c_5$.
If $c_j^1$ and $c_j^2$ do not enclose one another, we say that $c_j^1$ and 
$c_j^2$ are {\em parallel}, written as $c_j^1 \parallel_{x_i} c_j^2$ or 
simply $c_j^1 \parallel c_j^2$ In Fig.~\ref{fig:mpsat6}, $c_1 \parallel c_2$. 

Let $c_j^1, \ldots, c_j^k$ be all the clauses containing the variable 
$x_i$ and are arranged from left to right in the given planar 
embedding. We observe that for $1 \le \ell < k$, $c_j^{\ell}$ and 
$c_j^{\ell +1}$ are adjacent, and $c_j^{\ell}$ is parallel 
to $c_j^{\ell +1}$, encloses $c_j^{\ell + 1}$, or is enclosed by 
$c_j^{\ell + 1}$. Moreover, 
if $c_j^{\ell} \parallel c_j^{\ell +1}$, then it must be the case that 
($c_j^{\ell}$ may be $c_j^1$ and $c_j^{\ell + 1}$ may be $c_j^k$)
\begin{align}\label{eq:r1}
c_j^1 \pl \ldots \pl c_j^{\ell} \parallel c_j^{\ell + 1} \sl \ldots \sl c_j^k.
\end{align}
It is also possible that ($c_j^{\ell}$ may be $c_j^1$ or $c_j^k$)
\begin{align}\label{eq:r2}
c_j^1 \pl \ldots \pl c_j^{\ell} \sl \ldots \sl c_j^k.
\end{align}

It is not hard to see that for any variable, either~\eqref{eq:r1} 
or~\eqref{eq:r2} must happen. 

\subsection{Gadget Design}
For the NP-hardness reduction, variable and clause gadgets are created. 
The clause gadget is relatively simple with an example given in 
Fig.~\ref{fig:clause-gadget} where the blue object represents the clause 
gadget for $c_2 = x_1 \vee x_2$. In the figure, the left and right partially 
drawn lime objects correspond to portions of the variable gadgets for 
$x_1$ and $x_2$, respectively. The clause object can be only be grasped 
from the top at one of the two thin handles indicated by the red arrows, 
after the corresponding lime object is removed first. In particular, 
it is not possible to remove the object from the lower side of the 
handles due to the long and thin gap, even after the lime objects are 
removed. We note that the figure is not drawn to scale; the relative 
dimensions of the gadget will be specified later. 
\begin{figure}[ht!]
\vspace*{-2mm}
\begin{center}
\begin{overpic}[width=2in,tics=5]{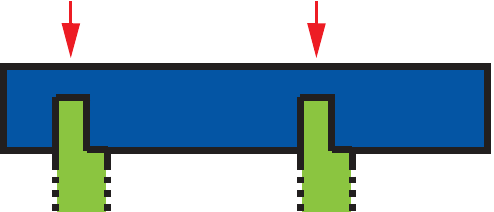}
\end{overpic} 
\end{center}
\vspace*{-2mm}
\caption{\label{fig:clause-gadget} The (blue) clause gadget for 
$c_1 = x_1 \vee x_2$. Figure is not drawn to scale.}
\vspace*{-2mm}
\end{figure} 

The variable gadget is more involved with several variations, depending 
on how the related clauses are arranged in the embedding of the \mpsat
instance. For a variable $x_i$, let $c_j^1, \ldots, c_j^k$ be all the 
positive clauses containing the positive literal $x_i$ and are arranged 
from left to right in the given planar embedding. 
The simplest case is when $c_j^1 \pl \ldots \pl c_j^k$ or 
$c_j^1 \sl \ldots \sl c_j^k$ holds and the same applies to the negative
clauses containing $\neg x_i$. In the given \mpsat instance, $x_1$ falls
into this category and the gadget for $x_1$ is given in 
Fig.~\ref{fig:variable-gadget}(a). The gadget contains a green positive 
{\em selector} object and an orange negative selector object that can only 
be lifted at the thin handles indicated by the red arrows. They can be 
lifted from either side of the handles when there are sufficient clearance
at a given side of a handle. Focusing on the positive side, there are two 
lime positive {\em connectors} that link the positive selector to (blue)
positive clause objects $c_3$ (top) and $c_1$ (bottom). These connectors 
can be removed from either ends by grasping the thin handles when they are 
exposed. In the figure, these thin handles are blocked by the selector 
object and clause objects. Similar construction are shown for the negative 
side with yellow negative connectors and cyan negative clauses (partially 
shown). 
\begin{figure}[ht!]
\begin{center}
\begin{overpic}[width=6.5in,tics=5]{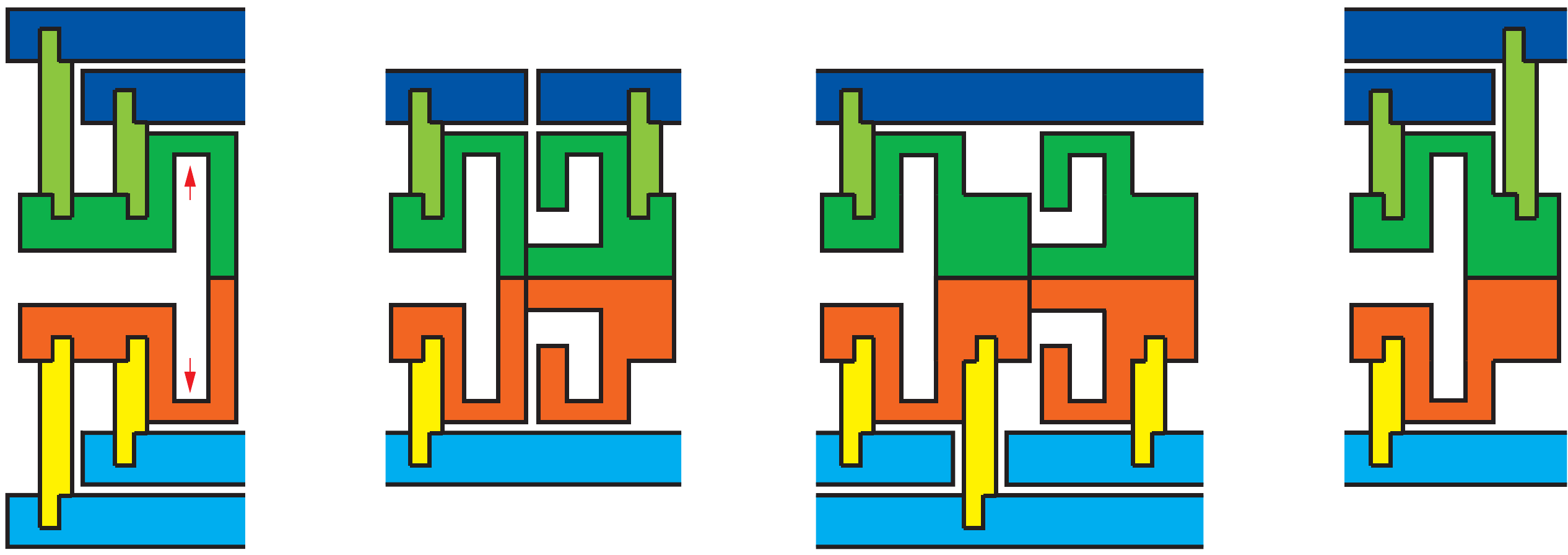}
\put(7,-3){{\small (a)}}
\put(32,-3){{\small (b)}}
\put(63,-3){{\small (c)}}
\put(91.5,-3){{\small (d)}}
\put(19,33){{\small $c_3$}}
\put(19,28.5){{\small $c_1$}}
\put(19,5.5){{\small $c_4$}}
\put(19,1.5){{\small $c_6$}}
\put(47,28.5){{\small $c_2$}}
\put(47,5.5){{\small $c_4$}}
\put(47,1.5){{\small $c_6$}}
\put(80,33){{\small $c_3$}}
\put(80,28.5){{\small $c_2$}}
\put(80,5.5){{\small $c_5$}}
\end{overpic} 
\end{center}
\vspace*{-2mm}
\caption{\label{fig:variable-gadget} Variable gadgets (excluding the 
blue and cyan clause gadgets) for (a) $x_1$. (b) $x_2$. (c) $x_3$ (d) 
$x_4$. The figures are not drawn to scale.}
\vspace*{-2mm}
\end{figure}
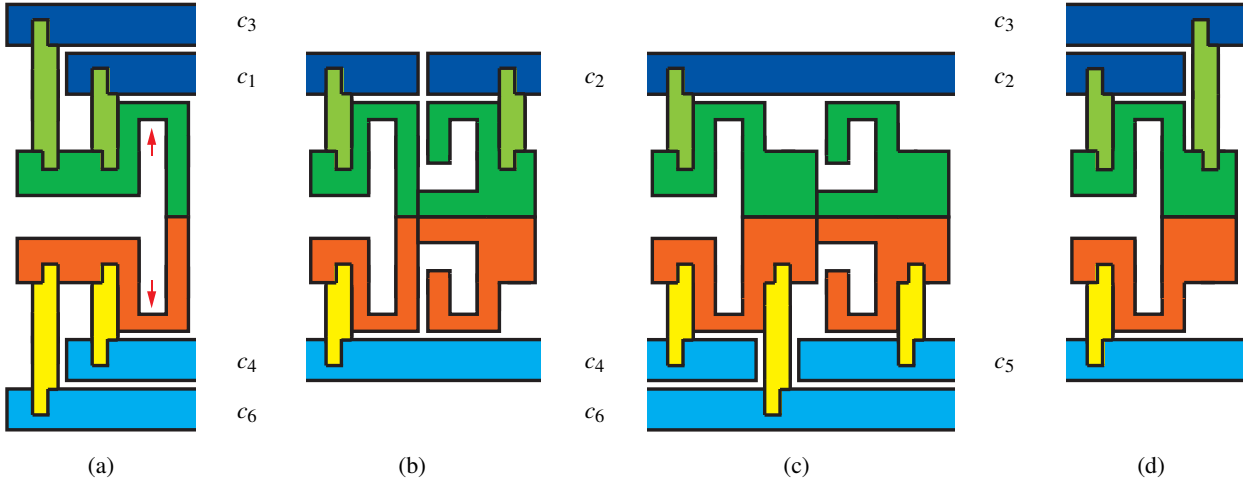 

We note that the top side of the thin handle on the connector is blocked 
by $c_1$ and subsequently by $c_3$, which reflects the fact of $c_3 \sl_{x_1}
c_1$. In general, if it is the case that $c_j^1 \pl \ldots \pl c_j^k$ for 
some variable $x_i$, then the corresponding selector is blocked by $c_j^1, 
\ldots c_j^k$ in that order. If instead $c_j^1 \sl \ldots \sl c_j^k$, then
the selector is blocked by $c_j^k, \ldots, c_j^1$ in that order (the 
positive part of the gadget for $x_5$ in Fig.~\ref{fig:variable-gadget}(d) 
shows how this case can be realized, which requires a ``split'' selector 
design to allow connector appear on two sides of the selector). 

A more complex case is when parallel clauses exist for a given variable, 
i.e., the scenarios described by~\eqref{eq:r1} happens. In this case, the 
corresponding variable gadget has four selectors, two positive and two 
negative. Variable $x_2$ in the given \mpsat instance falls into this 
case. The gadget for $x_2$ is shown in Fig.~\ref{fig:variable-gadget}(b), 
where there are two positive selectors, the top side of each is blocked by 
a different clause ($c_1$ and $c_2$ in this case). In the general case 
of~\eqref{eq:r1}, the left selector should be blocked by $c_j^1 \pl \ldots 
\pl c_j^{\ell}$ in that order and the right selector is blocked by $c_j^k, 
\ldots, c_j^{\ell + 1}$ in that order. For $x_2$, even though the negative 
part falls into the simplest case, we still need to create two selectors 
to mirror the structure on the positive side. We note that in the setup, 
for a robot to go from left to right, it must remove either both positive 
selectors or both negative selectors. Removing one positive selector and 
one negative selector (see  Fig.~\ref{fig:variable-gadget}(b)) will not 
allow the robot to pass through the gadget from left to right. 

A slight more involved case is when the scenario described by~\eqref{eq:r2} 
happens. This case may be viewed as~\eqref{eq:r2} with one additional 
variable-clause connection inserted in the middle. For the case, we need 
to further modify the previous variable gadget construction to add the 
connector for $c_j^{\ell}$ in the middle. Fig.~\ref{fig:variable-gadget}(c)
shows such a scenario for variable $x_3$, of which the negative part 
matches the setting described by~\eqref{eq:r2}. 

We again note that the figures are not drawn to scale. In particular, 
the vertical span of the selectors (the distance between the two thin
handles) and connectors is the dominating length; all other dimensions 
including the horizontal span of the gadgets, the width of the horizontal 
corridors between the non-connected parts of the positive and the negative 
selectors, are negligible in comparison. 

\subsection{Hardness Proof}
The full construction of the \crp instance for a the given \mpsat 
instance from Fig.~\ref{fig:mpsat6} is provided in 
Fig.~\ref{fig:reduction-se-te}. Beside objects from the variable 
and clause gadgets, the only additional object is the black border 
object which has the following utility: {\em (i)} it can only be 
removed by grasping the thin handle in the middle from the left, 
after one of the last variable gadget selector object is removed, 
and {\em (ii)} it must be removed before any clause gadgets can be 
removed. The red hexagons indicates three exits; the distance between 
the top exit and the middle exit is $w$. Same is true for the lower 
two exits. As mentioned previously, the figure is not drawn to scale. 
The main dimension is the vertical span $2w$, which comes mainly 
from the connectors and the selectors, i.e., they are very long and 
thin. All other dimensions, e.g., the horizontal span of the entire
construction, are negligible when compared with $w$. It is assumed 
that the robot is just small enough to travel through the corridors 
between the variable selector objects and then go up or down to remove 
a selector by grasping the horizontal thin handle. Because of the 
assumption that the objects are all very thin in all but one dimension, 
they can be transported through the exits when grasped at one end 
(the black object can be split into two to allow easy transportation). 
A more dimension accurate sketch of the \crp instance is provided in 
Fig.~\ref{fig:reduction-se-te-flat}. The robot starts at the middle exit.
\begin{figure}[ht!]
\begin{center}
\begin{overpic}[width=6.5in,tics=5]{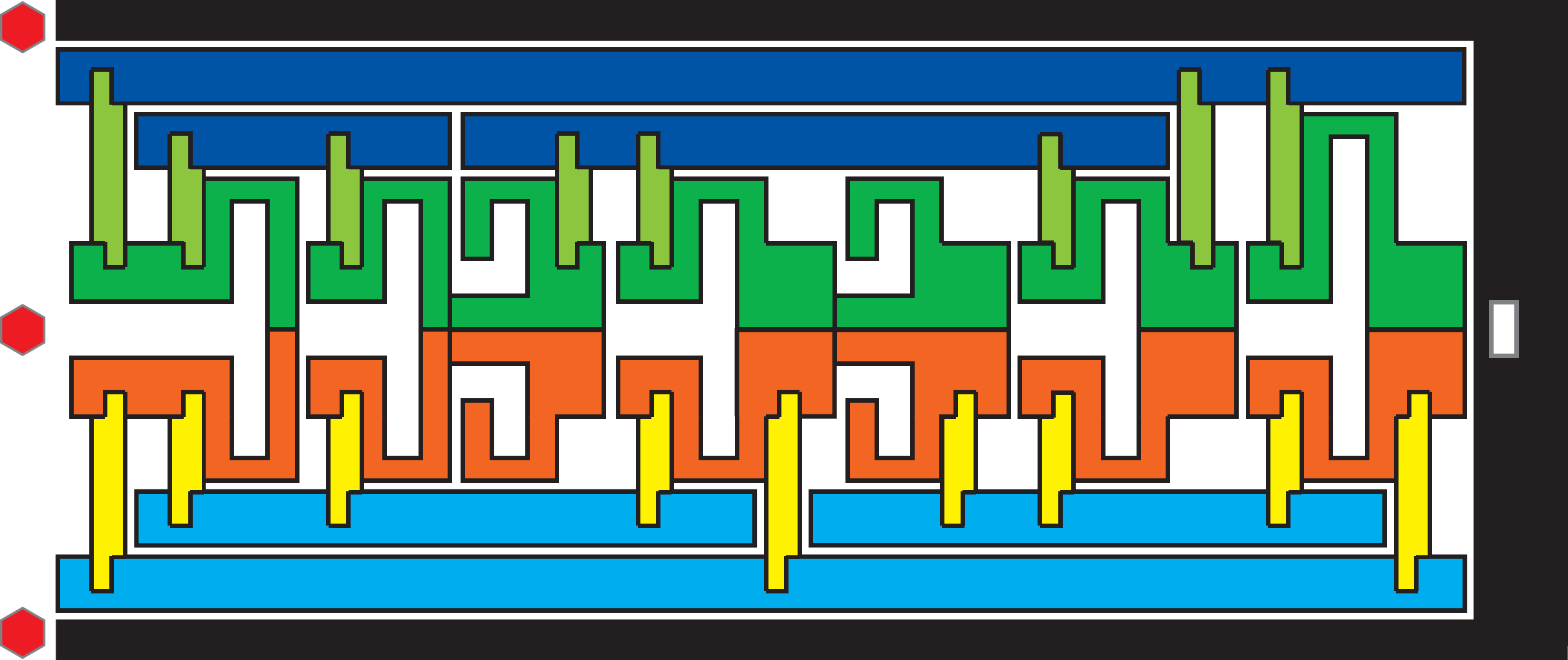}
\end{overpic} 
\end{center}
\caption{\label{fig:reduction-se-te} The \crp instance from the \mpsat 
instance given in Fig.~\ref{fig:mpsat6}. In the middle are the five variable 
gadgets for $x_1$--$x_5$, from left to right, with small vertical gaps 
in between. On the top are the three (blue) positive clause gadgets for 
$c_1$--$c_3$. On the bottom are the three (cyan) negative clause gadgets 
$c_4$--$c_6$. The variable and clause gadget positions largely mirror those
in the \mpsat instance. The three red hexagons mark the three exits. The 
robot is initially located at the middle exit. The figure is intended to 
show the necessary details and is not drawn to scale.}
\end{figure} 

\begin{figure}[ht!]
\begin{center}
\begin{overpic}[width=6.5in,tics=5]{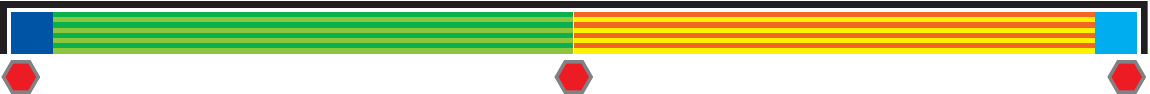}
\end{overpic}
\end{center}
\vspace*{-3mm}
\caption{\label{fig:reduction-se-te-flat} A more accurate sketch
of the \crp instance from Fig.~\ref{fig:reduction-se-te}, rotated $\pi/2$ 
counterclockwise.}
\end{figure}

The \crp instance is clearly feasible: all selectors can be removed 
sequentially through the middle exit, after which all connectors can be 
removed. Then, the black object can be removed, following by the removal 
of all clause gadgets. We now show that the the minimum possible cost is 
$(2n' + 3)w$ where $2n'$ is the number of selectors created for the 
\crp instance with $n \le n' < 2n$ (recall that some variables gadgets 
have two selectors and some have four selectors). For a given \mpsat 
instance with a fixed embedding, $n'$ can be readily computed. 

\begin{lemma}\label{l:ocr-cost-te-se}
The \crp instance admits a minimum possible travel cost of $(2n' + 3)w$.
\end{lemma}
\begin{proof}
First, given a feasible assignment to the \mpsat instance, we show 
that the \crp instance admits a solution with a total travel distance 
of $(2n' + 3)w$. Starting in the middle, if a variable $x_i$ is assigned 
to be positive (resp., negative), the robot removes the green (resp., 
orange) selector(s) from the $i$-th variable gadget from the left, 
which incurs a distance cost of $2w$ (one selector) or $4w$ (two positive 
selectors or two negative selectors) per variable. Doing this for all $n$ 
variable gadgets costs $2n'w$. All the exposed connectors can be 
removed from the middle with negligible cost. 
Then, the black border object can be removed with negligible cost (recall
that the horizontal span of Fig.~\ref{fig:reduction-se-te} is negligible 
in comparison to the vertical span $w$). 

At this point, because the assignment is feasible, every clause gadget 
object has at least one thin handle exposed. Next, the robot travels to 
the top (with a cost of $w$) and removes all (blue) positive clause 
gadgets through the top exit with negligible cost. Afterward, if there 
are positive parts (selectors and connectors) from variable gadgets, 
these can be removed as well through the left exit with negligible cost, 
because all positive connectors and selectors are exposed on the top 
after all positive clause gadgets are removed. Finally, the robot travels 
to the bottom (with a cost of $2w$) and remove everything that is left in 
a similar manner. The grand total is $(2n' + 3)w$.

Next, we show the cost is minimal. Because the removal of clause objects 
can only be done after the black object is removed, we have to first 
reach the black object. This is only possible by going from left to right 
and remove one selector from each top-bottom pair of selectors. This incurs 
a cost of $2n'w$. Then the robot must travel to top and then bottom at 
least once each, incurring an additional cost of $3w$. ~\qed
\end{proof}

\begin{theorem}\label{t:ocr-np-hard-te-se}
Planar optimal clutter removal is NP-hard for three exits. 
\end{theorem}
\begin{proof}
The proof of Lemma~\ref{l:ocr-cost-te-se} already shows that a solution 
to the \mpsat instance leads to a \crp solution of cost $(2n'+3)w$; we 
only need to prove the other direction. 
Assume that the constructed \crp problem has an optimal solution with a 
total cost of $(2n'+3)w$, we note that as the robot starts in the middle, 
it has to travel to top and bottom at least once to be able to remove the 
clause objects; this incurs a cost of $3w$. This leaves a budget of $2n'w$, 
which is the least possible for removing $n'$ selectors and then the 
black object. In turn, this means that the robot can travel to top and 
bottom at most once each. Therefore, it must remove all positive clause
gadgets during the single trip to the top (after removing the black object) 
and must remove all negative clause gadgets during the single trip to 
the bottom. Now, to be able to remove all positive (resp., negative) 
clause gadgets, one of the connectors to each of the positive (resp., 
negative) clause gadgets must be removed first, which can only be done 
after the corresponding variable selectors have been removed. Therefore, 
the set of $n'$ selectors that are removed must yield a satisfiable 
solution to the \mpsat instance. ~\qed
\end{proof}

Since it is easy to verify whether a given solution is optimal, planar 
\crp is also in NP. Therefore, optimal clutter removal is NP-complete. 

Using largely the same construction with minor modifications and
additional reasoning (e.g., we have not used all properties from the 
gadget design), we can show that optimally solving \crp is NP-complete
for a single exit. To do this, we only need to show the problem is NP-hard.

\begin{theorem}
Planar optimal clutter removal is NP-hard for the single exit case. 
\end{theorem}
\begin{proof} 
After first reducing an \mpsat instance to a \crp instance with three 
exits, we note that the structure, as sketched in 
Fig.~\ref{fig:reduction-se-te-flat}, can be ``folded'' at the middle 
so that the right side is folded over to be above the left side, as 
illustrated in Fig.~\ref{fig:reduction-se-te-fold}. For this instance, 
we set the only exit to be on the lower left as shown. The only 
important distance here is the horizontal span of the construction, which 
we assume to be $w$. 

\begin{figure}[ht!]
\begin{center}
\begin{overpic}[width=5in,tics=5]{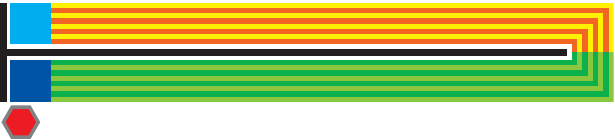}
\end{overpic}
\end{center}
\caption{\label{fig:reduction-se-te-fold} A ``folded'' \crp instance based 
on Fig.~\ref{fig:reduction-se-te-flat} (and Fig.~\ref{fig:reduction-se-te}).}
\end{figure} 

We claim that the minimum possible cost associated with the instance is now
$(4n' + 2m + 2)w$ when there is a satisfiable \mpsat solution. To see that 
this is true, we may first remove $n'$ selectors based on the \mpsat 
solution, which incurs a cost of $4n'w$ (instead of $2n'w$ before because
the only exit is at the left). Following this, $m$ connectors can be 
removed from the right that will expose a thin handle on each clause 
gadget object (cost: $2mw$). After removing the black object (cost: $2w$),
the robot can then remove all clause objects from the left, incurring 
negligible additional cost. Then, the rest of the connectors can be 
removed from the left with negligible additional cost. Any leftover 
selectors can also be removed by grasping them from the left side, adding
negligible cost. The total cost is then $(4n' + 2m + 2)w$.

Showing that a solution with cost $(4n' + 2m + 2)w$ implies a solution 
to the \mpsat is a bit more complex. The key properties used in proving 
the three-exit case remain, i.e., $n'$ variable selectors must first be 
removed before the black object can be removed; the clause gadgets can 
only be removed after the black object is removed. 

These properties are however insufficient for establishing NP-hardness 
for the single exit case. This is true because in the three-exit case, 
traveling between the different exit incur additional costs, which can 
be exploited to demonstrate hardness in finding optimal solutions. This 
is no longer the case for a single exit: it may be possible that the 
robot first removes a random set of $n'$ selectors (cost: $4n'w$) to allow 
it to remove the black object (cost: $2w$). Then, the robot may 
``iteratively'' remove connectors (from right) and then the clause objects 
on the left which takes negligible cost. Such removals may expose the 
handles of some selectors, which can be removed from the left. This could 
subsequently lead to the removal of connectors and clauses even though the
clauses are not satisfied by the initial $n'$ selector choices. 

To show that the iterative removal of the objects is not possible with 
a cost $(4n' + 2m + 2)w$, we need to show that if the initial $n'$ selector 
that are removed before removing the black object does not yield a 
satisfiable assignment, then some clause gadgets cannot be removed 
without removing additional selectors from the inner side of the selector
handles, which adds extra cost. Without loss of generality, let a positive 
clause $c_j = x_j^1 \vee x_j^2 \vee x_j^3$ be such a clause, i.e., the 
positive selectors for $x_j^1$--$x_j^3$ are not among the first $n'$ 
selectors that are removed. If the clause gadget for $c_j$ is to be removed, 
then one of the three connectors to it must first be removed. This 
requires the removal of at least one of the positive selectors for 
$x_j^1$--$x_j^3$. Without loss of generality, suppose at some point 
the clause gadget for $c_j$ can be removed due to the removal of 
the positive selector of $x_j^1$ (let us consider for now the case 
where there is a single positive selector for $x_j^1$). To be able 
to remove the selector from the outer/left side, all clauses blocking
$x_j^1$ must first be removed. However, $c_j$ is such a clause that 
blocks the selector. This forms a mutual lock between $c_j$ and 
the positive selector for $x_j^1$, preventing either from being 
removed (without incurring an extra cost of $6w$, from removing a 
selector and a connector). 
For the case where $x_j^1$ have two positive selectors, we only need 
to pay attention to the selector that affects $c_j$. The same reasoning 
then again applies.

In other words, if we let $c_j^1, \ldots, c_j^k$ be the set of positive
clauses whose corresponding variables $x_i^1, \ldots, x_i^q$ are not 
among these variables whose positive selectors are removed. Then $c_j^1, 
\ldots, c_j^k$ and $x_i^1, \ldots, x_i^q$ are mutually dependent in the 
sense that if the robot does not remove some positive selectors of 
$x_i^1, \ldots, x_i^q$ from the inner/right side (which incurs additional 
cost of $6w$ per removal of selector and connector), none of them can 
be removed. Therefore, to realize a total cost of $(4n' + 2m + 2)w$,
the \crp solution must yield a satisfactory assignment for the \mpsat 
instance. ~\qed
\end{proof}

\else
\fi

\end{document}